\documentclass{article}

\usepackage[preprint,nonatbib]{neurips_2022}

\usepackage[utf8]{inputenc} 
\usepackage[T1]{fontenc}    
\usepackage{hyperref}       
\usepackage{url}            
\usepackage{booktabs}       
\usepackage{amsfonts}       
\usepackage{nicefrac}       
\usepackage{microtype}      
\usepackage{xcolor}         

\usepackage{amssymb}
\usepackage{mathtools}
\usepackage{amsthm}

\usepackage{threeparttable}

\usepackage{algorithm}
\usepackage{algorithmic}

\theoremstyle{definition}
\newtheorem{Definition}{Definition}[section]

\theoremstyle{remark}
\newtheorem{Remark}{Remark}[section]

\theoremstyle{plain}
\newtheorem{Theorem}{Theorem}[section]
\newtheorem{Corollary}[Theorem]{Corollary}
\newtheorem{Proposition}[Theorem]{Proposition}
\newtheorem{Lemma}[Theorem]{Lemma}
\newtheorem{Assumption}[Theorem]{Assumption}
\usepackage{pifont}

\usepackage{amsmath,amsfonts,bm}

\def\eqref#1{(\ref{#1})}

\def\1{\bm{1}}

\def\ve{{\bm{e}}}

\def\vg{{\bm{g}}}

\def\vx{{\bm{x}}}
\def\vy{{\bm{y}}}
\def\vz{{\bm{z}}}

\def\mI{{\bm{I}}}

\DeclareMathAlphabet{\mathsfit}{\encodingdefault}{\sfdefault}{m}{sl}
\SetMathAlphabet{\mathsfit}{bold}{\encodingdefault}{\sfdefault}{bx}{n}

\def\gO{{\mathcal{O}}}

\newcommand{\E}{\mathbb{E}}

\title{Normalized/Clipped SGD with Perturbation for Differentially Private Non-Convex Optimization}

\author{%
  Xiaodong Yang\thanks{This work is done while the first author works as an intern at Microsoft Research Asia. The first two authors contribute equally.} \\
  School of Gifted Young \\
  University of Science and Technology of China \\
  Hefei, Anhui, 230026 \\
  \texttt{yangxiaodong@mail.ustc.edu.cn} \\
  \And
  Huishuai Zhang \\
  Microsoft Research Asia \\
  Beijing, 100080 \\
  \texttt{huzhang@microsoft.com} \\
  \AND
  Wei Chen \\
  Institute of Computing Technology \\
  Chinese Academy of Sciences \\
  Beijing, 100190 \\
  \texttt{chenwei2022@ict.ac.cn} \\
  \And
  Tie-Yan Liu \\
  Microsoft Research Asia \\
  Beijing, 100080 \\
  \texttt{tyliu@microsoft.com} \\
}

\begin{document}

\maketitle

\begin{abstract}
By ensuring differential privacy in the learning algorithms, one can rigorously mitigate the risk of large models  memorizing sensitive training data. In this paper, we study two algorithms for this purpose, i.e., DP-SGD and DP-NSGD, which first clip or normalize \textit{per-sample} gradients to bound the sensitivity and then add noise to obfuscate the exact information. We analyze the convergence behavior of these two algorithms in the non-convex optimization setting with two common assumptions and achieve a rate $\mathcal{O}\left(\sqrt[4]{\frac{d\log(1/\delta)}{N^2\epsilon^2}}\right)$ of the gradient norm for a $d$-dimensional model, $N$ samples and $(\epsilon,\delta)$-DP, which improves over previous bounds under much weaker assumptions.  Specifically, we introduce a regularizing factor in DP-NSGD and show that it is crucial in the convergence proof and subtly controls the bias and noise trade-off. Our proof deliberately handles the per-sample gradient clipping and normalization  that are specified for the private setting.   Empirically, we demonstrate that these two algorithms achieve similar best accuracy while DP-NSGD is  comparatively easier to tune than DP-SGD and hence may help further save the privacy budget when accounting the tuning effort. 
\end{abstract}

\section{Introduction}

Modern applications of machine learning strongly rely on training models with sensitive datasets, including medical records, real-life locations, browsing histories and so on. These successful applications raise an unavoidable risk of privacy leakage, especially when large models are shown to be able to memorize training data \cite{carlini2020extracting}. Differential Privacy (DP) is a powerful and flexible framework \cite{dwork2006calibrating} to quantify the influence of each individual and reduce the privacy risk. Specifically, we study the machine learning problem in the formalism of minimizing \textit{empirical risk} privately:
\begin{equation}\label{eq:empirical risk}
    \min_{\vx\in\mathbb{R}^d}f(\vx)\triangleq\mathbb{E}_\xi[\ell(\vx,\xi)]=\frac{1}{N}\sum_{i=1}^N\ell(\vx,\xi_i),
\end{equation}
where the objective $f(\vx)$ is an average of losses evaluated at each data point. 

In order to provably achieve the privacy guarantee, one   popular algorithm is \textit{differentially private stochastic gradient descent} or DP-SGD for abbreviation, which clips per-sample gradients with a preset threshold and perturbs the gradients with Gaussian noise at each iteration. Formally, given a set of gradients $\{\vg^{(i)},i\in\mathcal{S}\subset[N]\}$ computed at some data points and a threshold $c>0$, a learning rate $\eta>0$ and a noise 
multiplier $\sigma$, the updating rule goes from $\vx$ to the following
\begin{equation}\label{eq:updating rule clipping}
    \vx^+=\vx-\eta\left(\frac{1}{|\mathcal{S}|}\sum_{i\in\mathcal{S}}\bar{h}^{(i)}\vg^{(i)}+\bar{\vz}\right),
\end{equation}
where $\bar{\vz}\sim\mathcal{N}(0,c^2\sigma^2\mI_d)$ is an isotropic  Gaussian noise and $\bar{h}^{(i)}=\min\{1,c/\Vert \vg^{(i)}\Vert\}$  is the per-sample clipping factor. Intuitively speaking, the per-sample clipping procedure controls the influence of one individual.  DP-SGD \cite{abadi2016deep} has made a benchmark impact in deep learning with differential privacy, which is also referred to as \textit{gradient perturbation} approach. It has been extensively studied from many aspects, e.g., convergence \cite{bassily2014differentially, yu2020gradient}, privacy analysis \cite{abadi2016deep}, adaptive clipping threshold \cite{asi2021private, andrew2019differentially,pichapati2019adaclip},  hyperparameter choices \cite{li2021large, papernot2020making, papernot2021hyperparameter, mohapatra2021role} and so forth \cite{bu2020deep,bu2021convergence,papernot2020tempered}.

Another natural option to achieve differential privacy is  \textit{normalized gradient with perturbation}, which we coin ``DP-NSGD''. It normalizes per-sample gradients to control individual contribution and then adds noise accordingly. The update formula is the same as \eqref{eq:updating rule clipping} except replacing $\bar{h}^{(i)}$ with a \textit{per-sample  normalization factor} 
\begin{equation}
    h^{(i)}=\frac{1}{r+\Vert \vg^{(i)}\Vert}, \label{eq:normalize-factor}
\end{equation}
and replacing $\bar{\vz}$ with $\vz\in\mathcal{N}(0,\sigma^2\mathbf{I}_d)$ since each sample's influence is normalized to be $1$. In \eqref{eq:normalize-factor}, we introduce a regularizer $r>0$, which not only addresses the issue of ill-conditioned division but also controls the bias and noise trade-off as we will see in the analysis.

An intuitive thought that one may favor DP-NSGD is that the clipping is hard to tune due to the changing statistics of the gradients over the training trajectory \cite{andrew2019differentially,pichapati2019adaclip}. In more details, the injected Gaussian noise $\eta \bar{\vz}$ in \eqref{eq:updating rule clipping} is proportional to the clipping threshold $c$ and this noise component would dominate over the gradient component $\sum \bar{h}^{(i)}\vg^{(i)}/|\mathcal{S}|$, when the gradients $\Vert \vg^{(i)}\Vert\ll c$ are getting small as optimization algorithm iterates, thus hindering the overall convergence. DP-NSGD aims to alleviate this problem by replacing $\bar{h}^{(i)}$ in \eqref{eq:updating rule clipping} with a \textit{per-sample gradient normalization factor} $h^{(i)}$ in \eqref{eq:normalize-factor}, thus enhancing the signal component $g^{(i)}$ when it is too small.

It is obvious that both clipping and normalization introduce \emph{bias}\footnote{Here \emph{bias} means that the expected descent direction differs from the true gradient $\nabla f$.} that might prevent the optimizers from converging \cite{chen2020understanding,zhao2021convergence}. Most of previous works on the convergence of DP-SGD \cite{bassily2014differentially,asi2021private,yu2020gradient} neglect
the effect of such biases by assuming a global gradient upper bound of the problem, which does not exist for the cases of deep neural network models.  Chen et al.\cite{chen2020understanding} have made a first attempt to understand gradient clipping, but their results strongly rely on a symmetric assumption which is not that realistic.

In this paper, we consider both the effect of per-sample normalization/clipping and the injected Gaussian perturbation in the convergence analysis. If properly setting the hyper-parameters, we achieve $\mathcal{O}\left(\sqrt[4]{\frac{d\log(1/\delta)}{N^2\epsilon^2}}\right)$ convergence rate of the gradient norm for the general non-convex objectives with  $d$-dimensional model, $N$ samples and $(\epsilon,\delta)$-DP,  under only two weak assumptions \cite{zhang2020why}, $(L_0, L_1)$-generalized smoothness and $(\tau_0,\tau_1)$-bounded gradient variance. These assumptions are very mild compared to the usual ones as they allow the smoothness coefficient and the gradient variance growing with the norm of gradient, which is widely observed in the setting of deep learning.

Our contributions are summarized as follows.
\begin{itemize}
    \item For the \emph{differentially private empirical risk minimization},  we establish the convergence rate of the DP-NSGD and the DP-SGD algorithms for general non-convex objectives with only an $(L_0, L_1)$-smoothness condition, and explicitly characterizes the bias of the per-sample clipping or normalization.
    
    \item For the DP-NSGD algorithm, we introduce a regularizing factor which turns out to be crucial in the convergence analysis and induces interesting trade-off between the bias due to normalization  and the decaying rate of our upper bound. 

    \item We identify one key difference in the proofs of the DP-NSGD and DP-SGD. As the gradient norm  approaches zero, DP-NSGD cannot guarantee the function value to drop along the expected descent direction, and introduce an non-vanishing term that depends on the regularizer and the gradient variance. 
    
    
    
    \item We evaluate their empirical performance on deep models with $(\epsilon,\delta)$-DP and show that both DP-NSGD and DP-SGD can achieve comparable accuracy but the former is easier to tune than the later.
\end{itemize}

After introducing our problem setup in Section \ref{sec:2}, we present the algorithms and theorems in Section \ref{sec:3} and show numerical experiments on vision tasks in Section \ref{sec:4}. We make concluding remarks in Section \ref{sec:conclusion}.

\begin{table*}[t]
\caption{Expected gradient norm bounds (the smaller, the better) for non-convex empirical risk minimization with/without $(\epsilon,\delta)$-DP guarantee. All the algorithms assume certain bound on the gradient noise, which may be different from one another.  Notations: $N, T$ and $d$ are  the number of samples, the number of iterations and the number of parameters,  respectively. All bounds should be read as $\mathcal{O}(\cdot)$ and $\log\frac{1}{\delta}$ is omitted. }
\label{tbl:comparison}
\vskip 0.15in
\begin{center}
\begin{small}

\begin{threeparttable}
\begin{sc}
\begin{tabular}{lcccr}
\toprule
 Algorithm & Condition & Bias handled & Bound \\
 \midrule
 SGD \cite{ghadimi2013stochastic} &  $L$-Smooth \& $V$-Bounded Variance & N/A  & $\gO\left(\frac{1}{\sqrt{T}}+\frac{\sqrt{V}}{\sqrt[4]{T}}\right)$     	\\
  
   Clipped SGD \cite{zhang2020why} &  $(L_0,L_1)$-Smooth \& $V$-Bounded Variance  &  $\surd$   & $\gO\left(\frac{V^2}{\sqrt{T}}+\frac{V^{3/2}}{\sqrt[4]{T}}\right)$  	 \\
  
\midrule

DP-NSGD\tnote{*} \cite{das2021dp}  &  $L$-Smooth \& quasar-Convex  &  $\surd$ & $\mathcal{O}\left(\frac{D_X \sqrt{d}}{N\epsilon}+\E_i\|\vx_i^* - \vx^*\|\right)$ \tnote{**}    	\\

DP-GD \cite{wang2019differentially} &  L-Smooth \& Bounded Gradient & $\times$ & $\tilde{\gO}\left(\frac{d}{\log N\epsilon^2}\right)$    \\

DP-(N)SGD \textbf{(Ours)}  & $(L_0,L_1)$-Generalized Smooth  & $\surd$ &$ \gO\left(\sqrt[4]{\frac{d}{N^2\epsilon^2}}\right)$    \\
\bottomrule
\end{tabular}
\end{sc}
\begin{tablenotes}
\footnotesize
\item[*] To be precise, \cite{das2021dp} studies a client-level DP optimizer in a federated setting. Their optimizer, DP-NormFedAvg, uses vanilla GD for each client and normalizes the contribution of every client. Sharing similar motivations with our centralized DP-NSGD, their contributions are roughly credited to DP-NSGD.
\item[**] Here $\Vert\vx_i^\ast-\vx^\ast\Vert$ measures heterogeneity via the distance between the $i$-th client's local minimizer $\vx_i^\ast$ to the global minimizer $\vx^\ast$, and $D_X\triangleq\Vert \vx_0-\vx^\ast\Vert$.
\end{tablenotes}

\end{threeparttable}

\end{small}
\end{center}
\vskip -0.1in
\end{table*}

\section{Problem Setup}\label{sec:2}

\subsection{Notations}
Denote the private dataset as $\mathbb{D}=\{\xi_i,1\leq i\leq N\}$. The loss $\ell(\vx,\xi)$ is defined for every model parameters $\vx\in\mathbb{R}^d$ and sample instance $\xi$. In the seuqel, $\Vert \vx\Vert$ is denoted as the $\ell_2$ norm of a vector $\vx\in\mathbb{R}^d$, without other specifications.
Our target is to minimize the empirical average loss \eqref{eq:empirical risk} satisfying $(\epsilon,\delta)$-differential privacy \cite{dwork2006our}. From time to time, we also use $\nabla\ell(\vx,\xi)$ to denote the gradient of $\ell(\cdot,\cdot)$ w.r.t. $\vx$ evaluated at $(\vx,\xi)$. We are given an oracle to have multiple stochastic unbiased estimates $\vg^{(i)}$ for the gradient $\nabla f(\vx)$ of objective at any point.

\begin{Definition}[$(\epsilon,\delta)$-DP]
A randomized mechanism $\mathcal{M}$ guarantees $(\epsilon,\delta)$-differentially privacy if for any two neighboring input datasets $\mathbb{D}\sim\mathbb{D}^\prime$ ($\mathbb{D}^\prime$ differ from $\mathbb{D}$ by substituting one record of data) and for any subset of output $S$ it holds that $\text{Pr}[\mathcal{M}(\mathbb{D})\in S]\leq e^\epsilon\text{Pr}[\mathcal{M}(\mathbb{D}^\prime)\in S]+\delta$.
\end{Definition}
Besides, we also define the following notations to illustrate the bound we derived. We write $f(\cdot)=\mathcal{O}(g(\cdot))$, $f(\cdot)=\Omega(g(\cdot))$ to denote $f(\cdot)/g(\cdot)$ is upper or lower bounded by a positive constant. We also write $f(\cdot)=\Theta(g(\cdot))$ to denote that $f(\cdot)=\Omega (g(\cdot))$ and $f(\cdot)=\Omega (g(\cdot))$. Throughout this paper, we use $\E$ to represent taking expectation over the randomness of optimization procedures, including drawing gradients estimates $\vg$ and adding extra Gaussian perturbation $\vz$, while $\mathbb{E}_k$ takes conditional expectation given $\vx_k$.

In this non-convex setting, we measure the utility of some algorithm via bounding the expected minimum gradient norm $\mathbb{E}\left[\min_{0\leq k<T}\Vert\nabla f(\vx_k)\Vert\right]$.
If we want to measure the utility via bounding function values, the extra convex condition or its weakened versions are essential.

\subsection{Assumptions on Smoothness and Variance}
\begin{Definition}\label{def: relaxed smoothness}
We say that a continuously differentiable function $f(\vx)$ is $(L_0,L_1)$-generalized smooth, if for all $\vx,\vy\in\mathbb{R}^d$, we have $\Vert \nabla f(\vx)-\nabla f(\vy)\Vert\leq (L_0+L_1\Vert\nabla f(\vx)\Vert)\Vert \vx-\vy\Vert$.
\end{Definition}

Firstly appearing in \cite{zhang2020why}, a similar condition is derived via empirical observations that the smoothness $\Vert\nabla^2 f(\vx)\Vert$ increases with the gradient norm $\Vert\nabla f(\vx)\Vert$ in training language models.
If we set $L_1=0$, then Definition \ref{def: relaxed smoothness} turns into the usually assumed $L$-smooth. This relaxed notion of smoothness, Definition~\ref{def: relaxed smoothness}, and lower bounded function value together constitute our first assumption below.

\begin{Assumption}\label{assumption: relaxed smoothness, lower boundedness}
We assume that $f(\vx)$ is $(L_0,L_1)$-generalized smooth, Definition \ref{def: relaxed smoothness}. We also assume $D_f\triangleq f(\vx_0)-f^\ast<\infty$ where $f^\ast=\inf_{\vx\in\mathbb{R}}f(\vx)$ is the global minimum of $f(\cdot)$.
\end{Assumption}
We do not assume an upper bound on $D_X\triangleq\|\vx_0-\vx^*\|$, which may scale with the model dimension.

Moreover, to handle the stochasticity in gradient estimates, we   employ the following \textit{almost sure} upper bound on the gradient variance as another assumption.

\begin{Assumption}\label{assumption: a.s. sampling noise bound}
For all $\vx\in\mathbb{R}^d$, $\mathbb{E}[\vg]=\nabla f(\vx)$. Furthermore, there exists $\tau_0>0$ and $0\leq\tau_1<1$, such that with probability $1$, it holds $\Vert \vg-\nabla f(\vx)\Vert\leq\tau_0+\tau_1\Vert\nabla f(\vx)\Vert$.
\end{Assumption}
Previously, \cite{zhang2020why} and \cite{zhang2020improved} obtained their results under the exact assumption of $\tau_1=0$. In comparison, Assumption~\ref{assumption: a.s. sampling noise bound} is much weaker, as it allows the deviation $\Vert \vg-\nabla f(\vx)\Vert$ grows   with respect to the gradient norm $\Vert\nabla f(\vx)\Vert$, matching practical observation.

The \textit{almost surely} error bound sharply controls $h=1/(r+\Vert \vg\Vert)$ with probability $1$ by
\begin{equation*}
    \frac{1}{r+\tau_0+(1+\tau_1)\Vert\nabla f(\vx)\Vert}\leq h\leq\frac{1}{r-\tau_0+(1-\tau_1)\Vert\nabla f(\vx)\Vert}.
\end{equation*}
Analysis in \cite{DBLP:journals/corr/abs-2110-12459} can work with an \textit{expectation} version of Assumption \ref{assumption: a.s. sampling noise bound}: $\mathbb{E}\left[\Vert \vg-\nabla f(\vx)\Vert^2\right]\leq\tau_0+\tau_1\Vert\nabla f(\vx)\Vert^2$ because  they are allowed to control the error $\Vert\vg-\nabla f(\vx)\Vert$ via momentum in the non-private context.

In the non-private case, one usually chooses a mini-batch of data   $\vg=\dfrac{1}{|\mathcal{S}|}\sum_{i\in\mathcal{S}}\nabla f(\vx,\xi_i)$. Thereafter, one can reduce the variance factors $(\tau_0,\tau_1)$ with large batch size $|\mathcal{S}_k|$. However, in our private setting, clipping or normalization should apply on \textit{per-sample} gradients, indicating that we cannot effectively reduce the variance via large batch size. 

\subsection{More Related Works}\label{subsec:1.1}

\noindent
{\bf Private Deep Learning:}  

Many papers \cite{chaudhuri2011differentially, wang2017differentially, wang2019differentially, kuru2020differentially, yu2020gradient, wang2022differentially,asi2021private} have made attempts to theoretically analyze gradient perturbation approaches in various settings, including (strongly) convex or non-convex objectives. However, these papers did not take gradient clipping into consideration, and simply treat DP-SGD as SGD with extra Gaussian noise. Chen et al.\cite{chen2020understanding} made a first attempt to understand gradient clipping, but their results strongly rely on a symmetric assumption which is considered as unrealistic. Very recently, Zhang et al.\cite{zhang2021understanding} studied the convergence of DP-FedAvg with clipping, the federated averaging algorithm with differential privacy guarantee.

As for algorithms involving \textit{normalizing}, Das et al.\cite{das2021dp} studied DP-FedAvg with normalizing, which normalizes each client's update to unit-norm vector and then conducts the usual Fed-Avg operation. Their convergence analysis is based on one-point/quasar convexity and $L$-smoothness.

All these mentioned results are hard to compare due to the differences of the settings, assumptions and algorithms. We only present a part of them in Table~\ref{tbl:comparison}.

\vspace{2mm}
\noindent
{\bf Non-Convex Stochastic Optimization:} 
 Ghadimi and Lan\cite{ghadimi2013stochastic} established the convergence of randomized SGD for  non-convex optimization. 
 The objective is assumed to be $L$-smooth and the randomness on gradients is assumed to be light-tailed with factor $V$. We note that the rate $\gO(V/\sqrt[4]{T})$ has been shown to be optimal in the worst-case under the same condition \cite{arjevani2019lower}.

Outside the privacy community, understanding gradient normalization and clipping is also crucial in analyzing adaptive stochastic optimization methods, including AdaGrad \cite{duchi2011adaptive}, RMSProp \cite{hinton2012neural}, Adam \cite{ADAM} and normalized SGD \cite{cutkosky2020momentum}. However, with the average of a mini-batch of gradient estimates being clipped, this \textit{batch} gradient clipping differs greatly from the \textit{per-sample} gradient clipping in the private context.
Zhang et al.\cite{zhang2020why} and Zhang et al.\cite{zhang2020improved} showed the superiority of batch gradient clipping with and without momentum respectively under $(L_0,L_1)$-smoothness condition for non-convex optimization. Due to a strong connection between clipping and normalization, we also assume this relaxed condition in our analysis. We further explore this condition for some specific cases in great details. Cutkosky and Mehta\cite{cutkosky2021high} found that a fine integration of clipping, normalization and momentum, can overcome heavy-tailed gradient variances via a high-probability bound.
Jin et al.\cite{DBLP:journals/corr/abs-2110-12459} discovered that normalized SGD with momentum is also distributionally robust.

\section{Normalized/Clipped Stochastic Gradient Descent with Perturbation}\label{sec:3}

\subsection{Algorithms and Their Privacy Guarantees}
\begin{algorithm}[tb]
   \caption{Differentially Private Normalized Stochastic Gradient Descent, DP-NSGD}
   \label{alg:normalized SGD}
\begin{algorithmic}
   \STATE {\bfseries Input:} initial point $\vx_0$; number of epochs $T$; default learning rates $\eta_k$; mini batch size $B$; noise multiplier $\sigma$; regularizer $r$.
   \FOR{$k=0$ {\bfseries to} $T-1$}
   \STATE Draw a mini-batch $\mathcal{S}_k$ of size $B$ and compute individual gradients $\vg_k^{i}$  at point $\vx_k$ where $i\in\mathcal{S}_k$.
\STATE For $i\in\mathcal{S}_k$, compute per-sample normalizing factor {\small$$h_k^{(i)}=\frac{1}{r+\Vert \vg_k^{(i)}\Vert}.$$}
\STATE Draw $\vz_k\sim\mathcal{N}(0,\sigma^2\mI_d)$ and update the parameters by {\small $$\vx_{k+1}=\vx_k-\eta_k\left(\frac{1}{B}\sum_{i\in\mathcal{S}_k}h_k^{(i)}\vg_k^{(i)}+ \vz_k\right).$$}
   \ENDFOR
\end{algorithmic}
\end{algorithm}
Since no literature formally displays DP-NSGD in a centralized setting, we present it in Algorithm \ref{alg:normalized SGD}. Compared to the usual SGD update, DP-NSGD contains two more steps: per-sample gradient normalization, i.e., multiplying $\vg_k^{(i)}$ with $h_k^{(i)}$, and noise injection, i.e., adding $\vz_k$. The normalization well controls each sample's contribution to the update and the noise obfuscates the exact information.

The well-known DP-SGD \cite{abadi2016deep} replaces the normalization with clipping, i.e., replacing $h_k^{(i)}$ with $\bar{h}_k^{(i)}=\min\left\{1,c/\Vert \vg_k\Vert\right\}$
and replacing  $z_k$  with $\bar{z}_k\sim\mathcal{N}(0,c\sigma^2\mI_d)$ in Algorithm~\ref{alg:normalized SGD}. DP-SGD introduces a new hyper-parameter, the clipping threshold $c$.

To facilitate the common practice in private deep learning, we adopt \textit{uniform sub-sampling without sampling} for both theory and experiments, instead of \textit{Poisson sub-sampling} originally adopted in DP-SGD \cite{abadi2016deep}. Due to this difference, the following lemma shares the same expression as Theorem 1 in \cite{abadi2016deep}, but requires a new proof. Deferred in Appendix \ref{app:privacy}, this simple proof combines amplified privacy accountant by sub-sampling in \cite{bun2018composable} with the tight composition theorem for Renyi DP, \cite{mironov2019renyi}.

\begin{Lemma}[Privacy Guarantee]\label{thm:privacy guarantee}
Provided that $B<0.1N$, there exists absolute constants $c_1,c_2>0$ so that DP-SGD and DP-NSGD are $(\epsilon,\delta)$-differentially private for any $\epsilon<c_1B^2T/N^2$ and $\delta>0$ if we choose $\sigma\geq c_2\frac{B\sqrt{T\log(1/\delta)}}{N\epsilon}$.
\end{Lemma}

\subsection{Convergence Guarantee of DP-NSGD}

\begin{Theorem}\label{thm:main convergence}
Suppose that the objective $f(\vx)$ satisfies Assumption \ref{assumption: relaxed smoothness, lower boundedness} and \ref{assumption: a.s. sampling noise bound}. Given any noise multiplier $\sigma$ and a regularizer $r>\tau_0$, we run DP-NSGD (Algorithm \ref{alg:normalized SGD}) using constant learning rate
\begin{equation}\label{eq:set eta}
    \eta=\sqrt{\frac{2}{(L_1(r+\tau_0)+L_0)T d\sigma^2}},
\end{equation}
with sufficiently many iterations $T$ (larger than some constant determined by $(\sigma^2,d,L,\tau,r)$, as specified in Lemma \ref{lemma:verify condition general sigma}).
We can obtain the following upper bound on gradient norm
\begin{equation}
    \mathbb{E}\left[\min_{0\leq k<T}\Vert\nabla f(\vx_k)\Vert\right]\leq C\left(\sqrt[4]{\frac{(D_f+1)^2r^3d\sigma^2}{T}}+\sqrt[4]{\frac{1}{Tr^3d\sigma^2}}\right)+\frac{8(r+2\tau_0)\tau_0^2}{r(r+\tau_0)(1-\tau_1)^3}, 
\end{equation}
where $C$ is a constant depending on $(\tau_0, \tau_1)$ and $(L_0, L_1)$.
\end{Theorem}
Theorem \ref{thm:main convergence} is a general convergence for normalized SGD with perturbation. To achieve $(\epsilon, \delta)$-differential privacy, we can choose specific noise multiplier $\sigma$ and running iterations $T$.

\begin{Corollary}\label{cor:trade-off}
Under the same conditions of Theorem \ref{thm:main convergence}, we use $\sigma=c_2B\sqrt{T\log\frac{1}{\delta}}/(N\epsilon)$  with $c_2$ from Lemma \ref{thm:privacy guarantee} and set $T\ge\gO(N^2\epsilon^2/(B^2r^3d \log\frac{1}{\delta}))$. If we have sufficiently many samples (larger than $L_1$ times some constant determined by $(\epsilon,\delta,d,L,\tau,r,B)$, as specified in Lemma \ref{lemma:verify condition}), there holds the following privacy-utility trade-off
\begin{equation}
    \mathbb{E}\left[\min_{0\leq k<T}\Vert\nabla f(\vx_k)\Vert\right]\leq C^\prime\sqrt[4]{\frac{dr^3\log(1/\delta)}{N^2\epsilon^2}}+\frac{8(r+2\tau_0)\tau_0^2}{r(r+\tau_0)(1-\tau_1)^3},
\end{equation}
where $C^\prime$ is a constant depending on $(\tau_0, \tau_1)$, $(L_0, L_1)$, $D_f$ and $B$.
\end{Corollary}

There are two major obstacles in proving this theorem. One is to handle the normalized gradients, which is solved by carefully using $r$ and dividing the range of $\Vert\nabla f(\vx_k)\Vert$ into two cases. The other is to handle the Gaussian perturbation $
\vz$, whose variance $\sigma^2$ could even grow linearly with $T$. This is solved by setting the learning rate $\eta$ proportionally to $1/\sigma$ in  \eqref{eq:set eta}. Combining the two steps together, we reach the  privacy-utility trade-off in Corollary \ref{cor:trade-off}.

\begin{proof}[Proof Sketch of Theorem~\ref{thm:main convergence}]
We firstly establish a descent inequality as in Lemma \ref{lemma:descent inequality} via exploiting the $(L_0,L_1)$-generalized smooth condition in Assumption \ref{assumption: relaxed smoothness, lower boundedness},
\begin{align*}
    &\mathbb{E}_k\left[f(\vx_{k+1})\right]-f(\vx_{k})\leq - \underbrace{\eta\mathbb{E}_k\left[\langle  h_k\nabla f(\vx_k),\vg_k \rangle\right]}_{\mathfrak{A}}\notag+ \underbrace{\frac{L_0+L_1\Vert \nabla f(\vx_k)\Vert}{2}\eta^2\left(d\sigma^2 + \mathbb{E}_k\left\Vert h_k\vg_k \right\Vert^2\right)}_{\mathfrak{B}}.
\end{align*}
In the above expression, we use $\mathbb{E}_k$ to denote taking expectation of $\{\vg_k^{(i)},i\in\mathcal{S}_k\}$ and $\vz_k$ conditioned on the past, especially $\vx_k$.
Next, in Lemma \ref{lemma:second order upper bound}, we upper bound the second order term $\mathfrak{B}$ by a constant $\mathcal{O}(\eta^2)$ plus a term like $\eta\mathbb{E}_k\left[h_k\Vert\nabla f(\vx_k)\Vert^2\right]$, which is compatible to $\mathfrak{A}$. In order to find simplified lower bound for $\mathfrak{A}$, we separate the time index $\{0,1,\cdots,T-1\}$ into two cases $\mathcal{U}:=\left\{0\leq k<T:\Vert\nabla f(\vx_k)\Vert\geq\frac{\tau_0}{1-\tau_1}\right\}$ and $\mathcal{U}^c$. Specifically, in Lemma \ref{lemma:first order lower bound}, we find that for $k\in\mathcal{U}$, the first order term $\mathfrak{A}$ is $\Omega(\eta\Vert\nabla f(\vx_k)\Vert)$ (see \eqref{eq:first order lower bound 1} in Appendix \ref{app:lemmas}); for $k\notin\mathcal{U}$, the first order term $\mathfrak{A}$ is $\Omega(\eta(\Vert\nabla f(\vx_k)\Vert^2/r-\tau_0^3/r^2))$ (see \eqref{eq:first order lower bound 2} in Appendix \ref{app:lemmas}). Then our result follows from summing up descent inequalities and scaling $\eta$ deliberately. 
\end{proof}

There are rich literature investigating the convergence properties of \textit{normalized gradient} methods in the non-private non-convex optimization setting. These results heavily rely on the following inequality to control the amount of descent
\begin{equation}\label{eq:deal roughly with h}
    -\left\langle\nabla f(\vx_k),\frac{\vg_k}{\Vert \vg_k\Vert}\right\rangle\leq-\Vert\nabla f(\vx_k)\Vert+2\Vert \vg_k-\nabla f(\vx_k)\Vert.
\end{equation}
Pitifully, based on this inequality, one unavoidably needs to control the error term $\Vert \vg_k-\nabla f(\vx_k)\Vert$ well to have the overall convergence. In practice, You et al.\cite{DBLP:journals/corr/abs-1904-00962} used large batch size $B\sim\mathcal{O}(T)$ to reduce the variance. However, this trick cannot apply in the private setting due to \textit{per-sample gradient processing}, e.g., clipping or normalization. In theory, Cutkosky and Mehta \cite{cutkosky2020momentum} and Jin et al.\cite{DBLP:journals/corr/abs-2110-12459} use momentum techniques with a properly scaled weight decay and obtain a convergence rate $\mathbb{E}\Vert\nabla f(\vx_T)\Vert=\mathcal{O}(1/\sqrt[4]{T})$, which is comparable with the usual SGD in the non-convex setup \cite{ghadimi2013stochastic}. However, momentum techniques do not apply well in the private setting either, because we   only have access to previous descent directions with  noise due to the composite differential privacy requirement. As far as we know, there is no successful application of momentum in the private community, either practically or theoretically.

In this paper, we view the regularizer $r$ as a tunable hyperparameter, and make our upper bound decay as fast as possible $\mathcal{O}(1/\sqrt[4]{T})$ by tuning $r$. However, due to the restrictions imposed by privacy protection, we are unable to properly employ large batch size or use momentum, thus leaving a strictly positive bound $\mathcal{O}(\tau_0^2/r)$ in the right hand side of Theorem \ref{thm:main convergence} and Corollary \ref{cor:trade-off}. Another observation is that $r$ trades off between the non-vanishing bound $\mathcal{O}(\tau_0^2/r)$ and the decaying term of $\mathcal{O}\left(\sqrt[4]{dr^3\log\frac{1}{\delta}/(N^2\epsilon^2)}\right)$.

\subsection{Convergence Guarantee of DP-SGD}\label{subsec:dpsgd convergence}
We now turn our attention to \textit{clipped} DP-SGD. Before we compare the distinctions, we present the following theorem on convergence.

\begin{Theorem}\label{thm:main convergence clip}
Suppose that the objective $f(\vx)$ satisfies Assumption \ref{assumption: relaxed smoothness, lower boundedness} and \ref{assumption: a.s. sampling noise bound}. Given any noise multiplier $\sigma>0$ and any clipping threshold $c>2\tau_0/(1-\tau_1)$, we run DP-SGD using constant learning rate
\begin{equation}\label{eq:set eta clip}
    \eta=\sqrt{\frac{2}{(L_1(c+\tau_0)+L_0)T dc^2\sigma^2}},
\end{equation}
with sufficiently many iterations $T$ (larger than some constant determined by $(\sigma^2,d,L,\tau,c)$ respectively, specified in Lemma \ref{lemma:verify condition general sigma clip}).
We can obtain the following upper bound on gradient norms
\begin{equation}
    \mathbb{E}\left[\min_{0\leq k<T}\Vert\nabla f(\vx_k)\Vert\right]\leq C\left(\sqrt[4]{\frac{(D_f+1)^2c^3d\sigma^2}{T}}+\sqrt[4]{\frac{1}{Tc^2(c+\tau_0)d\sigma^2}}\right), 
\end{equation}
where we employ a constant $C$ only depending on $(\tau_0,\tau_1)$ and $(L_0,L_1)$.
\end{Theorem}

Again, we would like to combine Theorem \ref{thm:main convergence clip} and Lemma \ref{thm:privacy guarantee} to have a full characterization.

\begin{Corollary}\label{cor:trade-off clip}
Under the same conditions of Theorem \ref{thm:main convergence clip}, we use $\sigma=c_2B\sqrt{T\log(1/\delta)}/(N\epsilon)$ with $c_2$ from Lemma \ref{thm:privacy guarantee} and set $T\ge\gO(N^2\epsilon^2/(B^2c^3d \log\frac{1}{\delta}))$. If we have sufficiently many samples (larger than $L_1$ times some constant determined by $(\epsilon,\delta,d,L,\tau,c,B)$, as specified in Lemma \ref{lemma:verify condition clip}), there holds the following privacy-utility trade-off
\begin{equation}
    \mathbb{E}\left[\min_{0\leq k<T}\Vert\nabla f(\vx_k)\Vert\right]\leq C^\prime\sqrt[4]{\frac{dc^3\log(1/\delta)}{N^2\epsilon^2}},
\end{equation}
where $C^\prime$ is a constant depending on $(\tau_0,\tau_1)$, $(L_0,L_1)$, $D_f$ and $B$.
\end{Corollary}

By comparing Corollary~\ref{cor:trade-off clip} and Corollary~\ref{cor:trade-off}, the most significant distinction of DP-SGD from DP-NSGD is that clipping does not induce a non-vanishing term $\mathcal{O}(\tau_0^2/r)$ as what we obtained in Corollary \ref{cor:trade-off}.
This distinction is because $\mathfrak{A}:=\eta\mathbb{E}\left[\langle  \nabla f,h\vg \rangle\right]$ and $\bar{\mathfrak{A}}:=\eta\mathbb{E}\left[\langle \nabla f,\bar{h}\vg \rangle\right]$ behave quite differently in some cases (see details in Lemmas \ref{lemma:first order lower bound} \& \ref{lemma:first order lower bound clip} of Appendix~\ref{app:lemmas}).

Specifically, when $\Vert\nabla f\Vert$ is larger than $\tau_0/(1-\tau_1)$, we know $\langle\nabla f,\vg \rangle\geq 0$. Therefore, the following ordering
\begin{equation}\label{eq:ordering}
    \frac{c}{c+\Vert\vg\Vert}\leq\min\left \{1,\frac{c}{\Vert\vg\Vert}\right\}\leq \frac{2c}{c+\Vert\vg\Vert}
\end{equation}
guarantees $\mathfrak{A}$ and $\bar{\mathfrak{A}}$ to be equivalent to  $\Omega(\eta\Vert\nabla f\Vert)$, where \eqref{eq:ordering} can be argued by considering two cases $c>\|\vg\|$ and $c\le \|\vg\|$ separately.
When the gradient norm $\Vert\nabla f\Vert$ is small, the inner-product $\langle\nabla f,\vg \rangle$ could be of any sign, and we can only have $\mathfrak{A}=\Omega\left(\eta\left(\Vert\nabla f\Vert^2/r-\tau_0^3/r^2\right)\right)$ and $\bar{\mathfrak{A}}=\Omega\left(\eta\left(\Vert\nabla f\Vert^2\right)\right)$ instead. As $\mathfrak{A}$ controls the amount of descent within one iteration for DP-NSGD, the non-vanishing term appears. 
Continued in Appendix~\ref{subsec:example}, we provide a toy example of the distribution of $\vg$ to further illustrate the difference of $\mathfrak{A}$ and $\bar{\mathfrak{A}}$. This toy example will also confim that we have been using an optimal bound $\mathfrak{A}=\Omega\left(\eta\left(\Vert\nabla f\Vert^2/r-\tau_0^3/r^2\right)\right)$ when $\Vert\nabla f\Vert$ is small.

The training trajectories of DP-NSGD fluctuate more adversely than DP-SGD, since $\mathfrak{A}$ can be of any sign while $\bar{\mathfrak{A}}$ stays positive. This difference is also observed empirically (see Figure \ref{fig:curves-clippedsgd-normedsgd}) that the training loss of normalized SGD  with $r=0.01$ (closer to normalization) fluctuates more than that of the clipped SGD with $c=1$\footnote{$c=1$ makes the magnitude of clipped SGD similar as normalized SGD and hence the comparison is more meaningful.}.

\subsection{On the Biases from Normalization and Clipping}

We discuss further on how gradient \emph{normalization} or \emph{clipping} affects the overall convergence of the private algorithms. The influence is two-folded: one is that clipping/normalization induces \emph{bias}, i.e., the gap between true gradient $\nabla f$ and clipped/normalized gradient; the other is that added Gaussain noise for privacy may scale with the regularizer $r$ and the clipping threshold $c$.

\textbf{The Induced Bias.}
When writing the objective as an empirical average
$f(\vx)=\sum_i\ell(\vx,\xi_i)/N$, the true gradient is $\nabla f(\vx)=\frac{1}{N}\sum_{i=1}^{N}\nabla\ell(\vx,\xi_i)$. Then both expected descent directions of Normalized SGD and Clipped SGD
\begin{equation*}
    \mathbb{E}[h\vg]=\frac{1}{N}\sum_{i=1}^N \nabla \ell(\vx,\xi_i)\frac{1}{r+\Vert\nabla\ell(\vx,\xi_i)\Vert},\mathbb{E}[\bar{h}\vg]=\frac{1}{N}\sum_{i=1}^N \nabla \ell(\vx,\xi_i)\min\left\{1,\frac{c}{\Vert\nabla\ell(\vx,\xi_i)\Vert}\right\},
\end{equation*}
deviate from the true gradient $\nabla f(\vx)$. This means that the normalization or the clipping induces biases compared with the true gradient.  A small regularizer $r$ or a small clipping threshold $c$ induces large \emph{biases}, and prevents the loss curves from dropping sharply while a large  $r$ or  $c$ could reduce such \emph{biases}.  This can be seen from the training loss curves of different values of $c$ and $r$ in Figure \ref{fig:curves-clippedsgd-normedsgd}, whose implementation details are in Section \ref{sec:4}. This matches the theoretical insight that the \emph{bias} itself hinders convergence.

 \begin{figure}
\centering
    \begin{minipage}{0.45\textwidth}
        \begin{center}
        \centerline{\includegraphics[width=90pt]{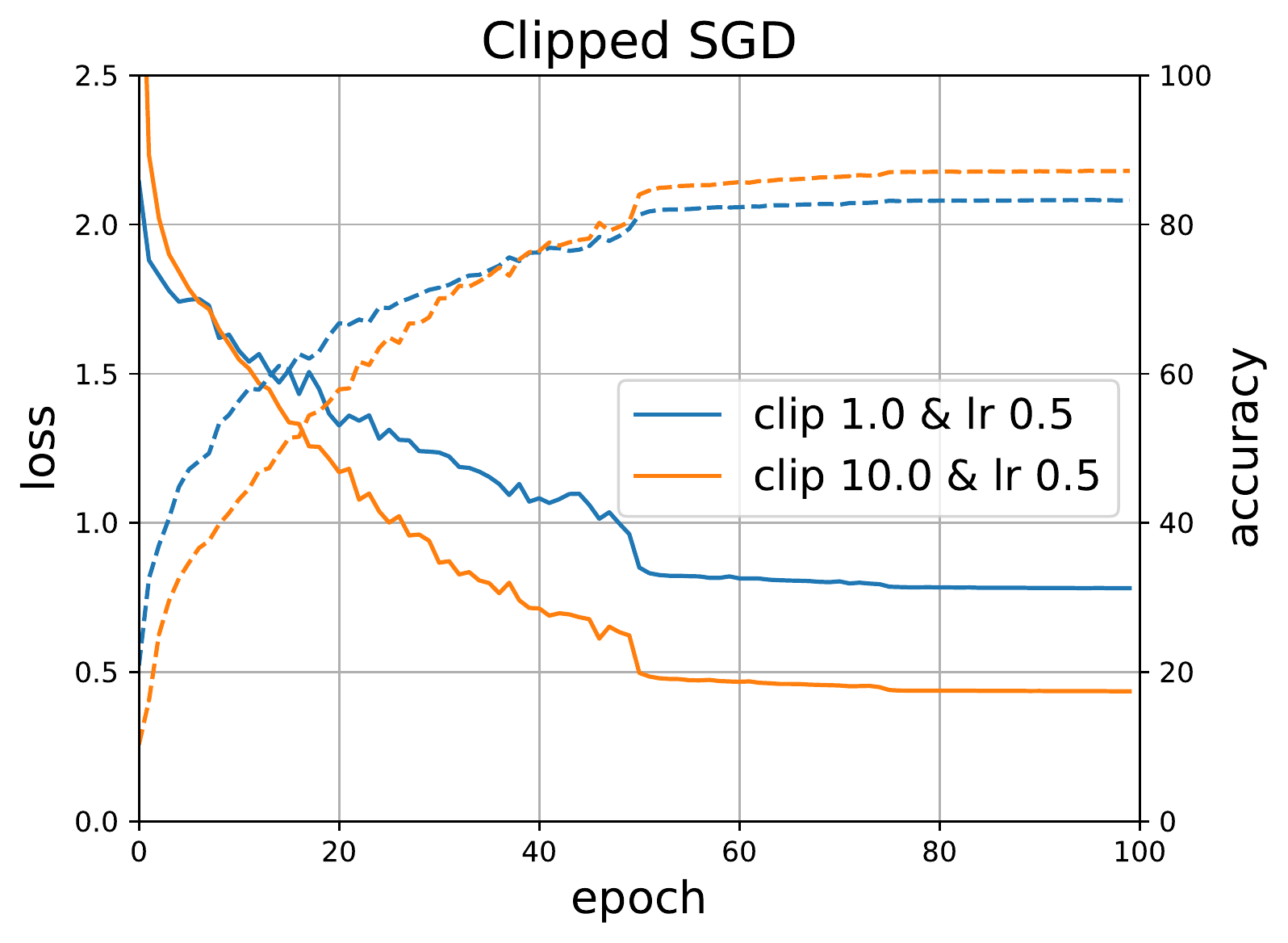}
        \includegraphics[width=90pt]{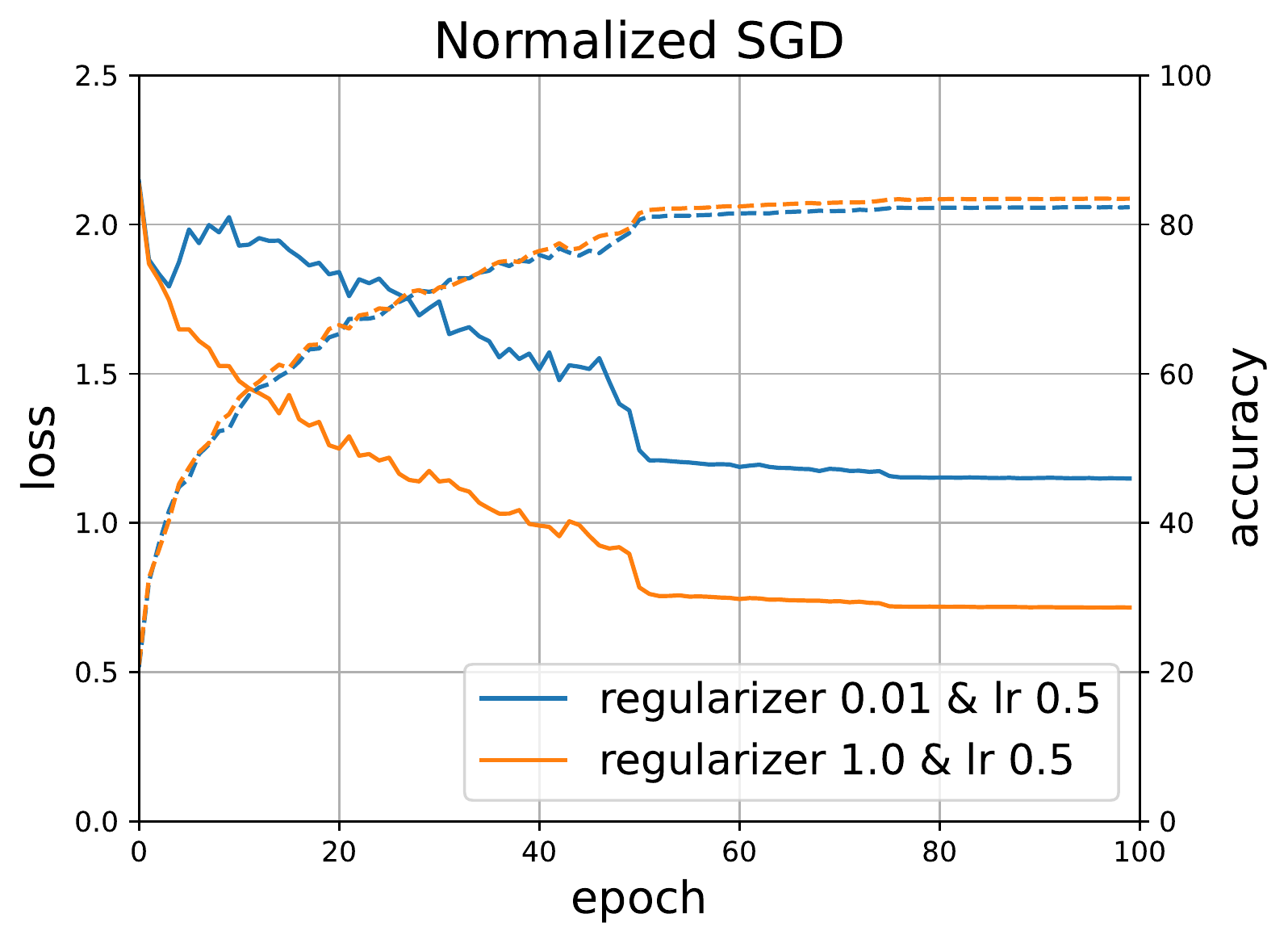}}
        \caption{Left: Training loss and training accuracy curves of Clipped SGD. Right:  Training loss and training accuracy curves of Normalized SGD. Both are trained with ResNet20 on CIFAR10 task.}
\label{fig:curves-clippedsgd-normedsgd}
        \end{center}
    \end{minipage}
    \hfill
    \begin{minipage}{0.45\textwidth}
        \begin{center}
        \centerline{\includegraphics[width=90pt]{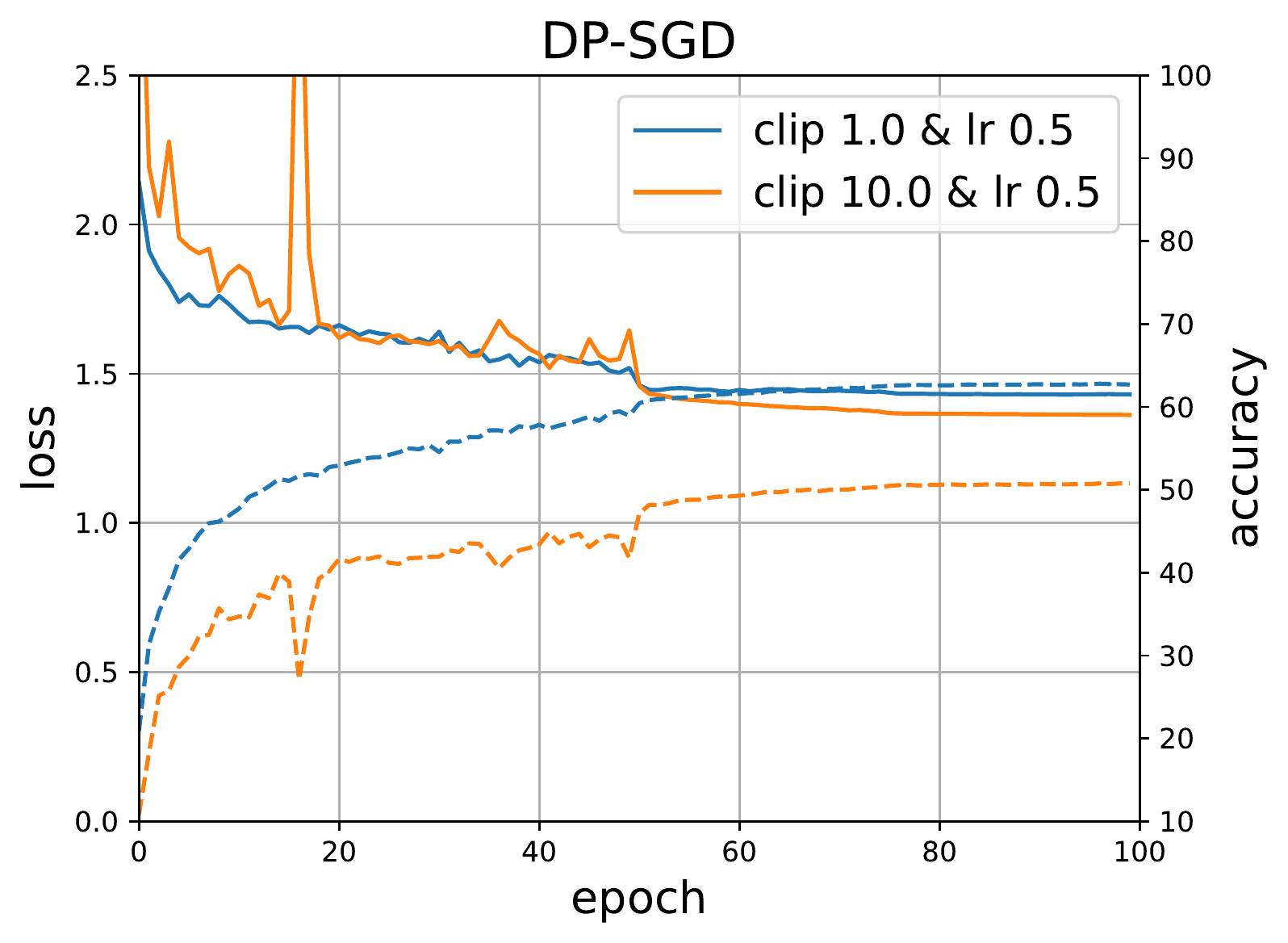}
        \includegraphics[width=90pt]{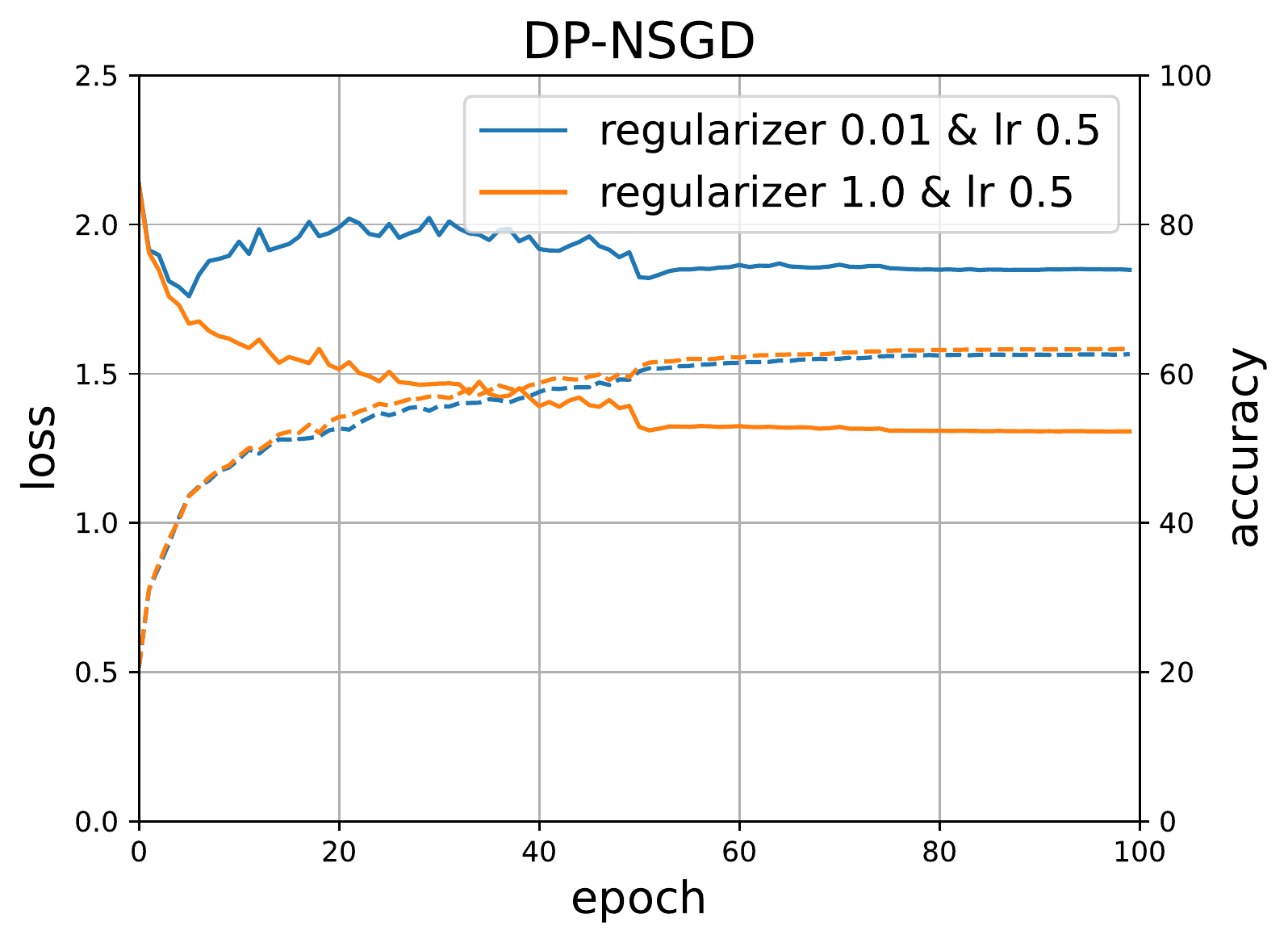}}
        \caption{ Left: Training loss and training accuracy curves of DP-SGD. Right:  Training loss and training accuracy curves of DP-NSGD. Task: ResNet20 on CIFAR10 with $\epsilon=8,\delta=\text{1e-5}$.}
\label{fig:curves-dpsgd-dpnsgd}
        \end{center}
    \end{minipage}
\vspace{-4mm}
\end{figure}

However surprisingly, the \emph{biases} affect  the accuracy curves differently for  clipped SGD and normalized SGD. The accuracy curves of clipped SGD vary with the value of $c$ while the value of $r$ makes almost no impact on the accuracy curves of normalized SGD. This phenomenon extends to the private setting (look at the accuracy curves in Figure \ref{fig:curves-dpsgd-dpnsgd}). 

A qualitative explanation would be as follows. After several epochs of training, \emph{good} samples $\xi$ (those already been classified correctly)  yield small gradients $\nabla\ell(\vx,\xi)$ and \emph{bad} samples $\xi^\prime$ (those not been correctly classified) yield large gradients $\nabla\ell(\vx,\xi^\prime)$. Typically, $c$ is set on the level of gradient norms, while $r$ is for regularizing the division. As the training goes, the gradient norm becomes small and clipping has no effect on the \emph{good} samples while normalized SGD comparatively strengthen the effect of \emph{good} samples. Therefore normalized SGD shows less drop in loss while obtaining a considerable level of accuracy. We call for a future investigation towards understanding this phenomenon  thoroughly. Specific to the setting $r\approx0,c=1$ in Figure~\ref{fig:curves-clippedsgd-normedsgd}, normalized SGD normalizes all $\nabla \ell(\vx,\xi)$ to a unit level, while clipped SGD would not change small gradients $\nabla \ell(\vx,\xi)$.

From a theoretical perspective, to give a finer-grained analysis of the \emph{bias}, imposing further assumptions to control $\gamma$ may be a promising future direction. For example, Chen et al.\cite{chen2020understanding} made an attempt towards this aspect, but their assumption that $\{\nabla\ell(\vx,\xi),\xi\in\mathbb{D}\}$ is nearly symmetric, is a bit artificial and not intuitive.  Sankararaman et al.\cite{sankararaman2020impact} proposed a concept \textit{gradient confusion}, defined as $\gamma=-\min\{\langle\nabla\ell(\vx,\xi_i),\nabla\ell(\vx,\xi_j)\rangle:i\neq j\}$ to approximately quantify how the per-sample gradients \emph{align} to each other.

\textbf{The Added Noises for Privacy Guarantee}. For the gradient \emph{clipping}, the added noise (Gaussian perturbation) $\bar{z}\sim\mathcal{N}(0,c^2\sigma^2\mI_d)$ is proportional to $c$, while for gradient \emph{normalizing}  $z\sim\mathcal{N}(0,\sigma^2\mI_d)$ keeps invariant with $r$. This suggests that when tuning DP-SGD, $\eta$ needs to vary when $c$ changes, in order to control the noise component $\eta\bar{z}$ in each update. In contrast, DP-NSGD is more likely to be robust under different scales of $r$, and thus is easier to tune  heuristically. Extensive experiments in Section \ref{sec:4} also support this claim empirically.

\section{Experiments}\label{sec:4}

This section conducts experiments to demonstrate the efficacy of Algorithm \ref{alg:normalized SGD} and compare the behavior of DP-SGD and DP-NSGD empirically. One example for the proof of concept is training a ResNet20 \cite{he2016deep} with CIFAR-10 dataset. As in literature \cite{yu2020gradient}, we replace all batch normalization layers with group normalization \cite{wu2018group} layers for easily computing the per-sample gradients.
The non-private accuracy for CIFAR-10 is 90.4\%. We compare the performances of DP-NSGD and DP-SGD with a wide range of hyper-parameters and different learning rate scheduling rules. All experiments can be run on a single Tesla V100 with 16GB memory. The ResNet20 has 270K  trainable parameters. 

\textbf{Hyperparameter choices.}
We first fix the privacy budget $\epsilon=\{2.0,4.0,8.0\}, \delta=10^{-5}$, which corresponds to setting the noise multiplier $\sigma=\{3.6, 2.0, 1.2\}$ for the case of batch size 1000 and number of epochs 100 with R\'enyi differential privacy accountant \cite{abadi2016deep,mironov2019renyi}. There are tighter privacy accountants \cite{gopi2021numerical} that can save $\epsilon$ for this noise multiplier. We then fix the weight decay to be 0 and use the classical learning rate scheduling strategy that multiplies the initial $lr$ with $0.1$ at epoch 50 and $0.01$ at epoch 75 respectively. The hyperparameters to tune are the initial learning rate $lr$ and the clip threshold $c$ for DP-SGD \cite{abadi2016deep}, where $lr$ takes values $\{0.05, 0.1, 0.2, 0.4, 0.8, 1.6, 3.2\}$ and $c$ takes values $\{0.1, 0.4, 1.6, 6.4, 12.8\}$. At the same time, the hyperparameters to tune for DP-NSGD are the initial learning rate $lr$ and the regularizer $r$, where $lr$ takes values $\{0.05, 0.1, 0.2, 0.4, 0.8, 1.6, 3.2\}$ and $r$ takes values $\{0.0001, 0.001, 0.01, 0.1, 1.0\}$. We compare the validation accuracy of DP-SGD and DP-NSGD via heatmaps of the above hyperparameter choices in Figure \ref{fig:heatmap-dpnsgd}.  We can see that the performance of DP-NSGD is rather stable for the regularizer taking values from $10^{-4}$ to $1.0$ and it is mostly affected by the learning rate. This is in sharp contrast with the case of DP-SGD where the performance depends on both the learning rate and the clip threshold in a complicated way. This indicates that  it is easy to tune the hyperparameters for DP-NSGD, which could not only reduce the tuning effort but also save the privacy budget for tuning hyperparameters \cite{tramer2020differentially, papernot2021hyperparameter}.

\begin{figure}[ht]
\vspace{-2mm}
\begin{center}
\centerline{\includegraphics[width=0.9\linewidth]{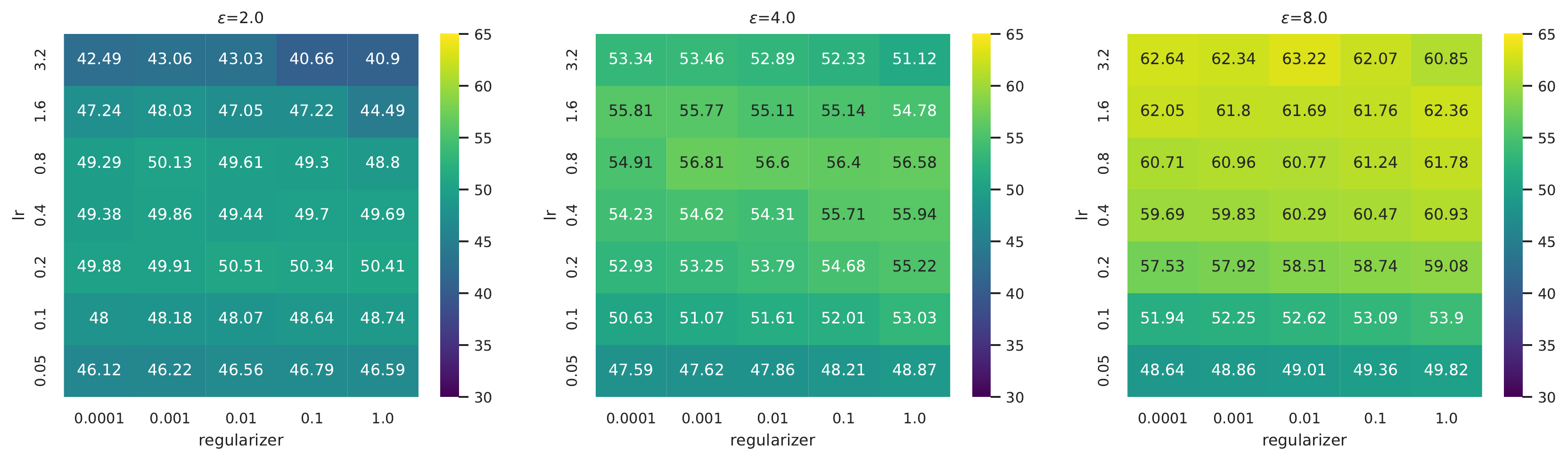}}
\centerline{\includegraphics[width=0.9\linewidth]{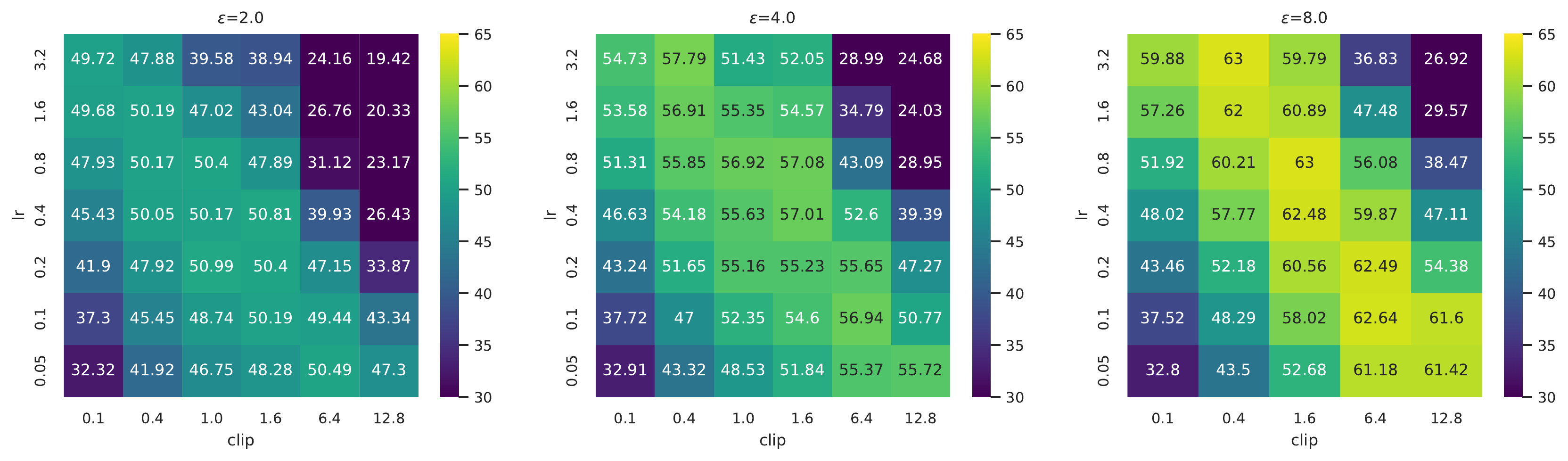}}
\caption{Experiments for ResNet20 on CIFAR10 task. Upper: Accuracy heatmap of DP-NSGD with varying lrs and regularizers. Lower: Accuracy heatmap of DP-SGD with varying lrs and clipping thresholds. The DP parameters are $\delta=1e^{-5}$ and  $\epsilon=2.0, 4.0, 8.0$ from left to right.}
\label{fig:heatmap-dpnsgd}
\end{center}
\vskip -0.2in
\end{figure}

We also run the above setting with the cyclic learning rate scheduling with $\text{min-lr}=0.02$ and $\text{max-lr}=1.0$. The best accuracy number are of DP-NSGD and DP-SGD can be as good as 66, which is comparable with the best number that is achieved  with model architecture modification in \cite{papernot2020tempered}. 

More experiments on language models are carried out in Appendix \ref{app:more-exp} and suggest similar observations.

\section{Concluding Remarks}\label{sec:conclusion}

In this paper, we have studied the convergence of two  algorithms, i.e., DP-SGD and DP-NSGD, for differentially private non-convex empirical risk minimization. We have achieved a rate that significantly improves over previous literature under similar setup and have analyzed the bias induced by the clipping or normalizing operation. 
As for future directions, it is very interesting to consider the convergence theorems under stronger assumptions on the gradient distribution.

\bibliographystyle{plain}

\clearpage

\appendix
\section{Prerequisite Lemmas}\label{app:lemmas}

The following is a standard lemma for $(L_0,L_1)$-generalized smooth functions, and it can be obtained via Taylor's expansion. Throughout this appendix, we use $\mathbb{E}_k$ to denote taking expectation of $\{\vg_k^{(i)},i\in\mathcal{S}_k\}$ and $\vz_k$ conditioned on the past, especially $\vx_k$.
\begin{Lemma}[Lemma C.4, \cite{DBLP:journals/corr/abs-2110-12459}]\label{lemma:Descent Lemma}
A function $f:\mathbb{R}^d\rightarrow\mathbb{R}^d$ is $(L_0,L_1)$-generalized smooth, then for any $\vx,\vx^+\in\mathbb{R}^d$,
\begin{equation*}
    f(\vx^{+})\leq f(\vx)+\left\langle \nabla f(\vx),\vx^+ - \vx \right\rangle+\frac{L_0+L_1\Vert \nabla f(\vx)\Vert}{2}\Vert \vx^+ - \vx \Vert^2.
\end{equation*}
\end{Lemma}

\begin{Lemma}\label{lemma:descent inequality}
For any $k\geq0$, we use $\vg_k$ to denote another realization of the underlying distribution behind the set of i.i.d. unbiased estimates $\{\vg^{(i)}_k:i\in\mathcal{S}_k\}$.
If we run DP-NSGD iteratively, the trajectory would satisfy the following bound:
\begin{equation}\label{eq:descent inequality 1}
    \mathbb{E}_k\left[f(\vx_{k+1})\right]-f(\vx_{k})\leq - \eta\mathbb{E}_k\left[\langle h_k\nabla f(\vx_k),\vg_k \rangle\right]+ \frac{L_0+L_1\Vert \nabla f(\vx_k)\Vert}{2}\eta^2\left(d\sigma^2 + \mathbb{E}_k\left\Vert h_k \vg_k \right\Vert^2\right).
\end{equation}
\end{Lemma}

\begin{proof}
The updating rule of our iterative algorithm could be summarized as
\begin{equation*}
    \vx_{k+1}=\vx_{k}-\eta\left(\frac{1}{B}\sum_{i\in\mathcal{S}_k}h^{(i)}_kg^{(i)}_k+ \vz_k\right),\quad \vz_k\sim\mathcal{N}(0,\sigma^2\mI_d).
\end{equation*}
By taking expectation $\mathbb{E}_k$ conditioned on the past, we rewrite the first-order term in Lemma \ref{lemma:Descent Lemma} into
\begin{equation}\label{eq:batch first}
    \mathbb{E}_k\left[\left\langle \nabla f(\vx),\vx_{k} - \vx_{k+1} \right\rangle\right]=\eta\mathbb{E}_k\left[\langle h_k\nabla f(\vx_k),\vg_k \rangle\right].
\end{equation}
In the same manner, we bound the second-order term by
\begin{equation}\label{eq:batch second}
    \mathbb{E}_k\Vert \vx_{k+1}-\vx_k\Vert^2=\eta^2 d \sigma^2 + \frac{\eta^2}{B^2} \mathbb{E}_k \left\Vert \sum_{i\in\mathcal{S}_k}h^{(i)}_k \vg^{(i)}_k \right\Vert^2\leq \eta^2 \left(d\sigma^2 + \mathbb{E}_k\left\Vert h_k \vg_k \right\Vert^2\right),
\end{equation}
where the last inequality follows from an elementary Cauchy-Schwarz inequality,
\begin{equation*}
    \left\Vert \sum_{i\in\mathcal{S}_k}h^{(i)}_k \vg^{(i)}_k \right\Vert^2\leq B\sum_{i\in\mathcal{S}_k}\left\Vert h^{(i)}_k \vg^{(i)}_k \right\Vert^2.
\end{equation*}
Plug \eqref{eq:batch first} and \eqref{eq:batch second} into Lemma \ref{lemma:Descent Lemma} to obtain the desired result.
\end{proof}

\begin{Remark}
This lemma implies that mini batch size $B$ does not affect expected upper bounds, due to \emph{per-sample} gradient normalization. We need to point out that $B$ could still influence high-probability upper bounds, and call for future investigations. 
\end{Remark}

In Section \ref{subsec:B.1}, we will upper bound the second-order term $\frac{L_0+L_1\Vert \nabla f(\vx_k)\Vert}{2}\eta^2\left(d\sigma^2 + \mathbb{E}_k\left\Vert h_k \vg_k \right\Vert^2\right)$ by a sum of $\alpha\eta h_k\Vert\nabla f(\vx_k)\Vert$ (for some $0<\alpha<1$) and another term of $\mathcal{O}(\eta^2)$ via a proper scaling of $\eta$. We firstly present the following lemma to provide a simplified lower bound for the first-order terms
\begin{equation}
    \eta\mathbb{E}_k\left[\langle h_k\nabla f(\vx_k),\vg_k \rangle\right]-\eta\alpha\mathbb{E}_{k}[h_k]\Vert\nabla f(\vx_k)\Vert^2.
\end{equation}

\begin{Lemma}[Lower bound first-order terms for normalizing]\label{lemma:first order lower bound}
Define a function $A:\mathbb{R}_+\rightarrow\mathbb{R}$ as
\begin{equation}
    A(s)=\begin{cases}
    \left(\dfrac{\tau_0}{r(1-\tau_1)+2\tau_0}-\dfrac{\alpha}{1-\tau_1}\right)s, & \text{if }s\geq\dfrac{\tau_0}{1-\tau_1};\\
    \dfrac{(1-\alpha)(1-\tau_1)}{r(1-\tau_1)+2\tau_0}s^2-\dfrac{4\tau_0^3}{r(r+\tau_0)(1-\tau_1)^3}, & \text{otherwise.}
    \end{cases}
\end{equation}
Then we have
\begin{equation*}
    \eta\mathbb{E}_k\left[\langle h_k\nabla f(\vx_k),\vg_k \rangle\right]-\eta\alpha\mathbb{E}_{k}[h_k]\Vert\nabla f(\vx_k)\Vert^2\geq\eta A(\Vert\nabla f(\vx_k)\Vert).
\end{equation*}
\end{Lemma}

\begin{proof}
We prove this lemma via separating the range of $\Vert\nabla f(\vx_k)\Vert$. When $\Vert\nabla f(\vx_k)\Vert\geq\tau_0/(1-\tau_1)$, then
\begin{align*}
    &\langle \nabla f(\vx_k),\vg_k \rangle=\Vert\nabla f(\vx_k)\Vert^2+\langle \nabla f(\vx_k),\vg_k-\nabla f(\vx_k) \rangle\\
    \geq& (1-\tau_1)\Vert\nabla f(\vx_k)\Vert^2-\tau_0\Vert\nabla f(\vx_k)\Vert\geq0,
\end{align*}
followed by
\begin{align}
    &\eta\mathbb{E}_k\left[\langle h_k\nabla f(\vx_k),\vg_k \rangle\right]-\eta\alpha\mathbb{E}_{k}[h_k]\Vert\nabla f(\vx_k)\Vert^2\notag\\
    =&\mathbb{E}_k\left[\frac{\eta\left\langle\nabla f(\vx_k),\vg_k \right\rangle}{r+\Vert \vg_k\Vert}\right]-\mathbb{E}_{k}\left[\frac{\alpha\eta}{(r+\Vert \vg_k\Vert)}\Vert\nabla f(\vx_k)\Vert^2\right]\notag\\
    \geq&\mathbb{E}_k\left[\frac{\eta\left\langle\nabla f(\vx_k),\vg_k \right\rangle}{r+\tau_0+(1+\tau_1)\Vert \nabla f(\vx_k)\Vert}\right]-\frac{\alpha}{1-\tau_1}\eta\Vert\nabla f(\vx_k)\Vert\notag\\
    =&\frac{\eta\left\Vert\nabla f(\vx_k)\right\Vert^2}{r+\tau_0+(1+\tau_1)\Vert \nabla f(\vx_k)\Vert}-\frac{\alpha}{1-\tau_1}\eta\Vert\nabla f(\vx_k)\Vert\notag\\
    \geq&\left(\frac{\tau_0}{r(1-\tau_1)+2\tau_0}-\frac{\alpha}{1-\tau_1}\right)\eta\Vert\nabla f(\vx_k)\Vert.\label{eq:first order lower bound 1}
\end{align}
When $\Vert\nabla f(\vx_k)\Vert< \tau_0/(1-\tau_1)$, we have $ \Vert \vg_k-\nabla f(\vx_k)\Vert\leq\tau_0+\tau_1\Vert\nabla f(\vx_k)\Vert\leq\tau_0/(1-\tau_1)$ as well. Then we decompose the first-order terms by
\begin{align*}
    &\eta\mathbb{E}_k\left[\langle h_k\nabla f(\vx_k),\vg_k \rangle\right]-\eta\alpha\mathbb{E}_{k}[h_k]\Vert\nabla f(\vx_k)\Vert^2\\
    =&(1-\alpha)\eta\mathbb{E}_k[h_k]\Vert\nabla f(\vx_k)\Vert^2+\mathbb{E}_k\left[\frac{\eta\left\langle\nabla f(\vx_k),\vg_k-\nabla f(\vx_k) \right\rangle}{r+\Vert \vg_k\Vert}\right].
\end{align*}
On one hand, we know
\begin{equation*}
    h_k=\frac{1}{r+\Vert\vg_k\Vert}\geq\frac{1}{r+\tau_0+(1+\tau_1)\Vert\nabla f(\vx_k)\Vert}\geq\frac{1-\tau_1}{r(1-\tau_1)+2\tau_0}.
\end{equation*}
On the other hand, we also have
\begin{align*}
    &\mathbb{E}_k\left[\frac{\eta\left\langle\nabla f(\vx_k),\vg_k-\nabla f(\vx_k) \right\rangle}{r+\Vert \vg_k\Vert}\right]\\
    =&\mathbb{E}_k\left[\frac{\eta\left\langle\nabla f(\vx_k),\vg_k-\nabla f(\vx_k) \right\rangle}{r+\tau_0+(1+\tau_1)\Vert\nabla f(\vx_k)\Vert}\right]+\mathbb{E}_k\left[\frac{\eta\left[(1+\tau_1)\Vert\nabla f(\vx_k)\Vert+\tau_0-\Vert\vg_k\Vert\right]\left\langle\nabla f(\vx_k),\vg_k-\nabla f(\vx_k) \right\rangle}{(r+\Vert \vg_k\Vert)(r+\tau_0+(1+\tau_1)\Vert\nabla f(\vx_k)\Vert)}\right]\\
    \geq&-\eta\frac{4\tau_0^3}{r(r+\tau_0)(1-\tau_1)^3}.
\end{align*}
Therefore, we have
\begin{align}
    &\eta\mathbb{E}_k\left[\langle h_k\nabla f(\vx_k),\vg_k \rangle\right]-\eta\alpha\mathbb{E}_{k}[h_k]\Vert\nabla f(\vx_k)\Vert^2\notag\\
    \geq&\eta\left[\frac{(1-\alpha)(1-\tau_1)}{r(1-\tau_1)+2\tau_0}\Vert\nabla f(\vx_k)\Vert^2-\frac{4\tau_0^3}{r(r+\tau_0)(1-\tau_1)^3}\right]\label{eq:first order lower bound 2}.
\end{align}
Combine both cases to derive the required lemma.
\end{proof}

We then establish similar results for gradient clipping based algorithms, whose updating rule is described by \eqref{eq:updating rule clipping}. Firstly, we mimic the proof of Lemma \ref{lemma:descent inequality} and derive without proof the following lemma.

\begin{Lemma}\label{lemma:descent inequality clip}
For any $k\geq0$, we use $\vg_k$ to denote another realization of the underlying distribution behind the set of i.i.d. unbiased estimates $\{\vg^{(i)}_k:i\in\mathcal{S}_k\}$. If we run DP-SGD iteratively, the trajectory would satisfy the following bound:
\begin{equation}\label{eq:descent inequality 1 clip}
    \mathbb{E}_k\left[f(\vx_{k+1})\right]-f(\vx_{k})\leq - \eta\mathbb{E}_k\left[\langle \bar{h}_k\nabla f(\vx_k),\vg_k \rangle\right]+ \frac{L_0+L_1\Vert \nabla f(\vx_k)\Vert}{2}\eta^2\left(dc^2\sigma^2 + \mathbb{E}_k\left\Vert \bar{h}_k \vg_k \right\Vert^2\right).
\end{equation}
\end{Lemma}

Then, we provide another lemma for clipping, similar to Lemma \ref{lemma:first order lower bound}.

\begin{Lemma}[Lower bound first-order terms for clipping]\label{lemma:first order lower bound clip}
Define a function $B:\mathbb{R}_+\rightarrow\mathbb{R}$ as
\begin{equation}
    B(s)=\begin{cases}
    \left(\dfrac{\tau_0c}{c(1-\tau_1)+2\tau_0}-\dfrac{\alpha}{1-\tau_1}\right)s, & \text{if }s\geq\dfrac{\tau_0}{1-\tau_1};\\
    (1-\alpha)s^2, & \text{otherwise.}
    \end{cases}
\end{equation}
If we take $c\geq 2\tau_0/(1-\tau_1)$, then
\begin{equation*}
    \eta\mathbb{E}_k\left[\langle \bar{h}_k\nabla f(\vx_k),\vg_k \rangle\right]-\eta\alpha\mathbb{E}_{k}[\bar{h}_k]\Vert\nabla f(\vx_k)\Vert^2\geq\eta B(\Vert\nabla f(\vx_k)\Vert).
\end{equation*}
\end{Lemma}

\begin{proof}
Recall that $\bar{h}_k$ is defined as
\begin{equation*}
    \bar{h}_k=\min\left\{1,\frac{c}{\Vert\vg_k\Vert}\right\}\geq\frac{c}{c+\Vert\vg_k\Vert}.
\end{equation*}
Again, we take the strategy of separting the range of $\Vert\nabla f(\vx_k)\Vert$. When $\Vert\nabla f(\vx_k)\Vert\geq\tau_0/(1-\tau_1)$, we know $\left\langle \nabla f(\vx_k),\vg_k \right\rangle\geq0$, followed by
\begin{align}
    &\eta\mathbb{E}_k\left[\langle \bar{h}_k\nabla f(\vx_k),\vg_k \rangle\right]-\eta\alpha\mathbb{E}_{k}[\bar{h}_k]\Vert\nabla f(\vx_k)\Vert^2\notag\\
    \geq&\mathbb{E}_k\left[\frac{\eta c\left\langle\nabla f(\vx_k),\vg_k \right\rangle}{c+\Vert \vg_k\Vert}\right]-\mathbb{E}_{k}\left[\frac{\alpha\eta}{\Vert \vg_k\Vert}\Vert\nabla f(\vx_k)\Vert^2\right]\notag\\
    \geq&\mathbb{E}_k\left[\frac{\eta c\left\langle\nabla f(\vx_k),\vg_k \right\rangle}{c+\tau_0+(1+\tau_1)\Vert \nabla f(\vx_k)\Vert}\right]-\frac{\alpha}{1-\tau_1}\eta\Vert\nabla f(\vx_k)\Vert\notag\\
    =&\frac{\eta c\left\Vert\nabla f(\vx_k)\right\Vert^2}{c+\tau_0+(1+\tau_1)\Vert \nabla f(\vx_k)\Vert}-\frac{\alpha}{1-\tau_1}\eta\Vert\nabla f(\vx_k)\Vert\notag\\
    \geq&\left(\frac{\tau_0c}{c(1-\tau_1)+2\tau_0}-\frac{\alpha}{1-\tau_1}\right)\eta\Vert\nabla f(\vx_k)\Vert\notag.
\end{align}
Otherwise, when $\Vert\nabla f(\vx_k)\Vert<\tau_0/(1-\tau_1)\leq (c-\tau_0)/(1+\tau_1)$ where the second inequality follows from $c\geq 2\tau_0/(1-\tau_1)$, we know
\begin{equation*}
    \Vert \vg_k\Vert\leq \tau_0+(1+\tau_1)\Vert\nabla f(\vx_k)\Vert\leq c.
\end{equation*}
Therefore in this case, $\bar{h}_k=1$ and
\begin{equation*}
    \eta\mathbb{E}_k\left[\langle \bar{h}_k\nabla f(\vx_k),\vg_k \rangle\right]-\eta\alpha\mathbb{E}_{k}[\bar{h}_k]\Vert\nabla f(\vx_k)\Vert^2=(1-\alpha)\eta\Vert\nabla f(\vx_k)\Vert^2.
\end{equation*}
Combine both cases to conclude the desired lemma.
\end{proof}

\subsection{Toy Example}\label{subsec:example}
Followed from the discussion in Subsection~\ref{subsec:dpsgd convergence}, a toy example is to presented here further illustrate the different behavior of $\mathfrak{A}:=\eta\mathbb{E}\left[\langle  \nabla f,h\vg \rangle\right]$ and $\bar{\mathfrak{A}}:=\eta\mathbb{E}\left[\langle \nabla f,\bar{h}\vg \rangle\right]$. Assign this simple distribution to $\ve\triangleq\vg-\nabla f$:
\begin{equation*}
    \mathbb{P}\left(\ve=\frac{\tau_0\nabla f}{\Vert\nabla f\Vert}\right)=\frac{1}{3},\;\;\;
    \mathbb{P}\left(\ve=-\frac{\tau_0\nabla f}{2\Vert\nabla f\Vert}\right)=\frac{2}{3}.
\end{equation*}

This distribution certainly satisfies Assumption \ref{assumption: a.s. sampling noise bound} with $\tau_1=0$. We calculate the explicit formula of $\mathfrak{A}$ for this toy example
\begin{equation*}
    \mathfrak{A}=\frac{\eta(\Vert\nabla f\Vert^3+(3r+\tau_0/2)\Vert\nabla f \Vert^2-\tau_0^2\Vert\nabla f\Vert/2)}{3(r+\tau_0+\Vert\nabla f\Vert)(r+\tau_0/2-\Vert\nabla f\Vert)}.
\end{equation*}
For $\Vert\nabla f(\vx_k)\Vert\leq\tau_0^2/(10r)$, we can compute $\mathfrak{A}<0$.
It implies that the function value may not even decrease along $\mathbb{E}[h\vg]$ in this case and the learning curves are expected to fluctuate adversely. 
This toy example also supports that the lower bound $\mathfrak{A}=\Omega\left(\eta\left(\Vert\nabla f\Vert^2/r-\tau_0^3/r^2\right)\right)$ we derived is optimal.
In contrast for the clipping operation, as long as $\Vert\nabla f\Vert\leq c-\tau_0$, we have $\bar{h}\equiv 1$ and therefore $\bar{\mathfrak{A}}=\eta\Vert\nabla f\Vert^2$.

\section{Proofs for DP-NSGD}\label{sec:B}
In this section, we provide a rigorous convergence theory for \textit{normalized stochastic gradient descent with perturbation}.
Unavoidable error between $\vg_k$ and $\nabla f(\vx_k)$ is a central distinction between stochastic and deterministic optimization methods. We begin with an explicit decomposition for \eqref{eq:descent inequality 1}

\begin{align}
    \mathbb{E}_k\left[f(\vx_{k+1})\right]-f(\vx_{k})\leq& -\eta\mathbb{E}_k\left[h_k\right]\Vert\nabla f(\vx_k)\Vert^2- \eta\mathbb{E}_k\left[\langle h_k\nabla f(\vx_k),\vg_k-\nabla f(\vx_k) \rangle\right]\notag\\
    &+ \frac{L_0+L_1\Vert \nabla f(\vx_k)\Vert}{2}\eta^2\left( \mathbb{E}_k\left[h_k^2\left\Vert \nabla f(\vx_k) \right\Vert^2\right]+ 2\mathbb{E}_k\left[h_k^2\langle \vg_k-\nabla f(\vx_k),\nabla f(\vx_k)\rangle \right]\right)\notag\\
    &+ \frac{L_0+L_1\Vert \nabla f(\vx_k)\Vert}{2}\eta^2\left(d\sigma^2+\mathbb{E}_k\left[h_k^2\left\Vert \vg_k-\nabla f(\vx_k) \right\Vert^2\right]\right).\label{eq:descent inequality 2}
\end{align}

\subsection{Upper Bound Second-order Terms}\label{subsec:B.1}
In theory, we need to carefully distinguish these terms accourding to their orders of $\eta$, as the first order term $\mathbb{E}_k\left[\langle h_k\nabla f(\vx_k),\vg_k \rangle\right]$ controls the amount of descent mainly. We show that the second order terms could be bounded by first order terms via a proper scaling of $\eta$, in the following technical lemma.

\begin{Lemma}\label{lemma:second order upper bound}
For any $0<\alpha<1$ to be determined explicitly later, if
\begin{equation}\label{eq:condition for second-order lemma}
    \eta\leq\min\left(\dfrac{(r-\tau_0)\alpha}{4L_0},\dfrac{(1-\tau_1)\alpha}{4L_1},\dfrac{\alpha}{6L_1d\sigma^2}\right)
\end{equation}
then we have
\begin{align}
    \frac{L_0+L_1\Vert\nabla f(\vx_k)\Vert}{2}\eta^2d\sigma^2\leq& \frac{L_0+L_1(r+\tau_0)}{2}\eta^2d\sigma^2+\frac{\alpha \eta  h_k}{4}\Vert\nabla f(\vx_k)\Vert^2,\label{eq:technical lemma 1}\\
    (L_0+L_1\Vert\nabla f(\vx_k)\Vert)\eta^2 h_k^2\langle \nabla f(\vx_k), \vg_k-\nabla f(\vx_k)\rangle \leq& \frac{(L_0(1-\tau_1)+L_1\tau_0)\tau_0^2}{r^2(1-\tau_1)^3}\eta^2+\frac{\alpha \eta  h_k}{4}\Vert\nabla f(\vx_k)\Vert^2,\label{eq:technical lemma 2}\\
    \frac{L_0+L_1\Vert\nabla f(\vx_k)\Vert}{2}\eta^2 h_k^2\Vert \vg_k-\nabla f(\vx_k) \Vert^2\leq& \frac{(L_0(1-\tau_1)+L_1\tau_0)\tau_0^2}{2r^2(1-\tau_1)^3}\eta^2 +\frac{\alpha \eta  h_k}{4}\Vert\nabla f(\vx_k)\Vert^2,\label{eq:technical lemma 3}\\
    \frac{L_0+L_1\Vert\nabla f(\vx_k)\Vert}{2}\eta^2 h_k^2\Vert\nabla f(\vx_k) \Vert^2\leq&\frac{\alpha \eta h_k}{4}\Vert\nabla f(\vx_k)\Vert^2.\label{eq:technical lemma 4}
\end{align}
\end{Lemma}

\begin{Remark}\label{remark:second order upper bound 1}
These bounds are proved by separating the range of $\Vert\nabla f(\vx_k)\Vert$. When it is smaller than some threshold, we can obtain an upper bound of $\mathcal{O}(\eta^2)$. Otherwise, when $\Vert\nabla f(\vx_k)\Vert$ is greater than the threshold, $h_k$ is of order $\Omega(1/\Vert\nabla f(\vx_k)\Vert)$, then the left hand terms are all of order $\mathcal{O}(\eta^2h_k\Vert f(\vx)\Vert^2)$. Therefore we scale $\eta$ small enough to make left hand terms smaller than $\alpha\eta h_k\Vert f(\vx)\Vert^2/4$. At last, we sum up the respective upper bounds together to conclude the lemma.
\end{Remark}

\begin{Remark}\label{remark:second order upper bound 2}
Moreover, we remark that the thresholds chosen during proof ($r+\tau_0$ for proving \eqref{eq:technical lemma 1} and $\tau_0/(1-\tau_1)$ for proving \eqref{eq:technical lemma 2} and \eqref{eq:technical lemma 3}) are quite artificial. A thorough investigation towards these thresholds would definitely improve the dependence on the constants $(L_0,L_1,\tau_0,\tau_1)$, but would not affect our main argument.
\end{Remark}

\begin{proof}
In fact, this lemma can be proved in an obvious way by separating into different cases.
\begin{description}
    \item[(i)] If $\Vert\nabla f(\vx_k)\Vert\leq r+\tau_0$, it directly follows that $L_1\Vert\nabla f(\vx_k)\Vert\eta^2d\sigma^2/2\leq L_1(r+\tau_0)\eta^2d\sigma^2/2$; otherwise, if $\Vert\nabla f(\vx_k)\Vert> r+\tau_0$, we know
    \begin{equation*}
        h_k=\frac{1}{r+\Vert \vg_k\Vert}\geq \frac{1}{r+\tau_0+(\tau_1+1)\Vert\nabla f(\vx_k)\Vert}\geq\frac{1}{3\Vert\nabla f(\vx_k)\Vert},
    \end{equation*}
    therefore $\eta\leq \alpha/(6L_1d\sigma^2)$ directly yields
    \begin{equation*}
        \frac{L_1\Vert\nabla f(\vx_k)\Vert\eta^2d\sigma^2}{2}\leq\frac{\eta\alpha h_k\Vert\nabla f(\vx_k)\Vert^2}{4}.
    \end{equation*}
    Then \eqref{eq:technical lemma 1} follows from summing up these two cases.
    \item[(ii)] If $\Vert\nabla f(\vx_k)\Vert\leq \tau_0/(1-\tau_1)$, then $\Vert \vg_k-\nabla f(\vx_k)\Vert\leq\tau_0+\tau_1\Vert\nabla f(\vx_k)\Vert\leq\tau_0/(1-\tau_1)$ and $h_k\leq1/r$, which yield
    \begin{align*}
        (L_0+L_1\Vert\nabla f(\vx_k)\Vert)\eta^2h_k^2\langle \nabla f(\vx_k), \vg_k-\nabla f(\vx_k)\rangle \leq& \frac{(L_0(1-\tau_1)+L_1\tau_0)\tau_0^2}{r^2(1-\tau_1)^3}\eta^2,\\
        \frac{L_0+L_1\Vert\nabla f(\vx_k)\Vert}{2}\eta^2h_k^2\Vert \vg_k-\nabla f(\vx_k) \Vert^2\leq& \frac{(L_0(1-\tau_1)+L_1\tau_0)\tau_0^2}{2r^2(1-\tau_1)^3}\eta^2.
    \end{align*}
    Otherwise, if $\Vert\nabla f(\vx_k)\Vert> \tau_0/(1-\tau_1)$, we note that $\Vert \vg_k\Vert\geq-\tau_0+(1-\tau_1)\Vert\nabla f(\vx_k)\Vert$ and
    \begin{equation*}
        h_k(L_0+L_1\Vert\nabla f(\vx_k)\Vert)\leq\frac{L_0+L_1\Vert\nabla f(\vx_k)\Vert}{r-\tau_0+(1-\tau_1)\Vert\nabla f(\vx_k)\Vert}\leq \max\left(\frac{L_0}{r-\tau_0},\frac{L_1}{1-\tau_1}\right).
    \end{equation*}
    Consequently, once $\eta\leq\dfrac{\alpha}{4}\min\left(\dfrac{r-\tau_0}{L_0},\dfrac{1-\tau_1}{L_1}\right)$, we have
    \begin{align*}
        &(L_0+L_1\Vert\nabla f(\vx_k)\Vert)\eta^2h_k^2\langle \nabla f(\vx_k), \vg_k-\nabla f(\vx_k)\rangle\\
        \leq& \max\left(\frac{L_0}{r-\tau_0},\frac{L_1}{1-\tau_1}\right)\eta^2 h_k\Vert\nabla f(\vx_k)\Vert^2\leq\frac{\eta\alpha h_k}{4}\Vert\nabla f(\vx_k)\Vert^2,
    \end{align*}
    and
    \begin{align*}
        &\frac{L_0+L_1\Vert\nabla f(\vx_k)\Vert}{2}\eta^2h_k^2\Vert \vg_k-\nabla f(\vx_k) \Vert^2\\
        \leq& \frac{1}{2}\max\left(\frac{L_0}{r-\tau_0},\frac{L_1}{1-\tau_1}\right)\eta^2 h_k\Vert\nabla f(\vx_k)\Vert^2\leq\frac{\eta\alpha h_k}{4}\Vert\nabla f(\vx_k)\Vert^2.
    \end{align*}
    We obtain \eqref{eq:technical lemma 2} and \eqref{eq:technical lemma 3} via summing up respective bounds for two cases.
    \item[(iii)] The last bound \eqref{eq:technical lemma 4} can be derived directly by
    \begin{equation*}
        \frac{L_0+L_1\Vert\nabla f(\vx_k)\Vert}{2}\eta^2h_k^2\Vert\nabla f(\vx_k) \Vert^2\leq\max\left(\frac{L_0}{r-\tau_0},\frac{L_1}{\tau_1}\right)\frac{\eta^2}{2} h_k\Vert\nabla f(\vx_k)\Vert^2\leq\frac{\eta\alpha h_k}{4}\Vert\nabla f(\vx_k)\Vert^2
    \end{equation*}
    via setting $\eta\leq\dfrac{\alpha}{2}\min\left(\dfrac{r-\tau_0}{L_0},\dfrac{\tau_1}{L_1}\right)$.
\end{description}
In conclusion, it suffices to set $\eta\leq\min\left(\dfrac{(r-\tau_0)\alpha}{4L_0},\dfrac{(1-\tau_1)\alpha}{4L_1},\dfrac{\alpha}{6L_1d\sigma^2}\right)$ to obtain these bounds.
\end{proof}

In the sequel, we will use this lemma only with
\begin{equation}\label{eq:set alpha}
    \alpha=\alpha_0:=\frac{\tau_0(1-\tau_1)}{2r(1-\tau_1)+4\tau_0}<\frac{1}{4}
\end{equation}

\begin{Lemma}\label{lemma:verify condition general sigma}
In the statement of Theorem \ref{thm:main convergence}, we take $\eta=\sqrt{\frac{2}{L_1(r+\tau_0)T d\sigma^2}}$. Then the condition \eqref{eq:condition for second-order lemma} in Lemma \ref{lemma:second order upper bound} holds as long as we run the algorithm long enough i.e. $T\geq C\left(\sigma^2,\tau,L,d,r\right)$.
\end{Lemma}

\begin{proof}
We see that
\begin{equation*}
    \eta=\sqrt{\frac{2}{L_1(r+\tau_0)T d\sigma^2}}\leq\min\left(\dfrac{(r-\tau_0)\alpha_0}{4L_0},\dfrac{(1-\tau_1)^2\alpha_0}{4L_1},\dfrac{\alpha_0}{6L_1d\sigma^2}\right)
\end{equation*}
is equivalent to
\begin{equation*}
    T\geq\max\left(\dfrac{32L_0^2}{(r-\tau_0)^2\alpha_0^2L_1(r+\tau_0)d\sigma^2},\dfrac{32L_1}{\tau_1^2\alpha_0^2(r+\tau_0)d\sigma^2},\dfrac{72L_1d}{\alpha_0^2(r+\tau_0)}\right).
\end{equation*}
\end{proof}

\subsection{Final Procedures in Proof}
\begin{proof}[Proof of Theorem \ref{thm:main convergence}]
Equipped with Lemmas \ref{lemma:first order lower bound} ,\ref{lemma:second order upper bound} and \ref{lemma:verify condition}, we further wrote the one step inequality \eqref{eq:descent inequality 2} into
\begin{align}
    \eta A(\Vert\nabla f(\vx_k)\Vert)\leq&f(\vx_{k})- \mathbb{E}_k\left[f(\vx_{k+1})\right]\notag\\
    &+\eta^2\left(\frac{(L_1(r+\tau_0)+L_0)d\sigma^2}{2}+ \frac{3(L_0(1-\tau_1)+L_1\tau_0)\tau_0^2}{2r^2(1-\tau_1)^3}\eta^2 \right).\label{eq:descent inequality 3}
\end{align}
We then separate the time index into
\begin{equation*}
    \mathcal{U}=\left\{k<T:\Vert\nabla f(\vx_k)\Vert\geq \frac{\tau_0}{1-\tau_1}\right\}
\end{equation*}
and $\mathcal{U}^c=\{0,1,\cdots,T-1\}\backslash\mathcal{U}$. Given this, we derive from \eqref{eq:first order lower bound 1} that for any $k\in\mathcal{U}$,
\begin{align*}
    &\frac{\tau_0}{2r(1-\tau_1)+4\tau_0}\eta\Vert\nabla f(\vx_k)\Vert\\
    \leq&\mathbb{E}_k\left[f(\vx_{k+1})\right]-f(\vx_{k})+\eta^2\left(\frac{(L_1(r+\tau_0)+L_0)d\sigma^2}{2}+ \frac{3(L_0(1-\tau_1)+L_1\tau_0)\tau_0^2}{2r^2(1-\tau_1)^3} \right).
\end{align*}
Similarly, together with $\alpha_0\leq1/4$, \eqref{eq:first order lower bound 2} deduces that for any $k\in\mathcal{U}^c$,
\begin{align*}
     &\eta\left[\frac{3(1-\tau_1)}{4r(1-\tau_1)+8\tau_0}\Vert\nabla f(\vx_k)\Vert^2-\frac{4\tau_0^3}{r(r+\tau_0)(1-\tau_1)^3}\right]\\
    \leq&\mathbb{E}_k\left[f(\vx_{k+1})\right]-f(\vx_{k})+\eta^2\left(\frac{(L_1(r+\tau_0)+L_0)d\sigma^2}{2}+\frac{3(L_0(1-\tau_1)+L_1\tau_0)\tau_0^2}{2r^2(1-\tau_1)^3} \right).
\end{align*}
Sum these inequalities altogether to have
\begin{align*}
    &\frac{1}{r+2\tau_0}\max\left\{\frac{\tau_0}{2T}\sum_{k\in\mathcal{U}}\Vert\nabla f(\vx_k)\Vert,\frac{3(1-\tau_1)}{4T}\sum_{k\notin\mathcal{U}}\Vert\nabla f(\vx_k)\Vert^2\right\}\\
    \leq&\frac{D_f}{T\eta}+\eta\frac{(L_1(r+\tau_0)+L_0)d\sigma^2}{2} +\eta \frac{3(L_0(1-\tau_1)+L_1\tau_0)\tau_0^2}{2r^2(1-\tau_1)^3} +\frac{4\tau_0^3}{r(r+\tau_0)(1-\tau_1)^3}\frac{|\mathcal{U}^c|}{T}.
\end{align*}
We set
\begin{equation}\label{eq:set eta updated}
    \eta=\sqrt{\frac{2}{(L_1(r+\tau_0)+L_0)T d\sigma^2}}.
\end{equation}
We then define
\begin{equation}
    \Delta=(D_f+1)\sqrt{\frac{(L_1(r+\tau_0)+L_0)d\sigma^2}{2T}}+ \frac{3(L_0+L_1\tau_0)\tau_0^2}{2r^2(1-\tau_1)^3}\sqrt{\frac{2}{(L_1(r+\tau_0)+L_0)T d\sigma^2}} +\frac{4\tau_0^3}{r(r+\tau_0)(1-\tau_1)^3}.
\end{equation}
Recall that $\sigma^2$ can have some dependence on $T$, so actually these three terms are $\mathcal{O}(\sigma/\sqrt{T}),\mathcal{O}(1/(\sqrt{T}\sigma))$ and $\mathcal{O}(1)$ respectively. Then we further have
\begin{align}
    &\mathbb{E}\left[\min_{0\leq k<T}\Vert\nabla f(\vx_k)\Vert\right]\notag\\
    \leq&\mathbb{E}\left[\min\left\{\sqrt{\dfrac{1}{|\mathcal{U}|}\sum_{k\in\mathcal{U}^c}\Vert\nabla f(\vx_k)\Vert^2},\dfrac{1}{|\mathcal{U}|}\sum_{k\notin\mathcal{U}}\Vert\nabla f(\vx_k)\Vert\right\}\right]\notag\\
    \leq&\max\left\{\sqrt{\dfrac{8(r+2\tau_0)}{3(1-\tau_1)}\Delta},\dfrac{2(r+2\tau_0)}{\tau_0}\Delta\right\}\label{eq:final inequality},
\end{align}
where the second inequality follows from the fact that either $|\mathcal{U}|\geq T/2$ or $|\mathcal{U}^c|\geq T/2$. In the end, we capture the leading terms in this upper bound to have
\begin{align*}
    &\mathbb{E}\left[\min_{0\leq k<T}\Vert\nabla f(\vx_k)\Vert\right]\leq\mathcal{O}\left(\sqrt[4]{\frac{(D_f+1)^2(L_1(r+\tau_0)+L_0)(r+2\tau_0)^2d\sigma^2}{T(1-\tau_1)^2}}\right)\\
    &+\mathcal{O}\left(\sqrt{\sqrt{\frac{2(r+2\tau_0)^2}{(L_1(r+\tau_0)+L_0)T d\sigma^2}}\frac{3(L_0+L_1\tau_0)\tau_0^2}{2r^2(1-\tau_1)^4}} +\frac{8(r+2\tau_0)\tau_0^2}{r(r+\tau_0)(1-\tau_1)^3}\right).
\end{align*}
\end{proof}

\begin{Lemma}\label{lemma:verify condition}
In the statement of Corollary \ref{cor:trade-off}, we take $\eta=\sqrt{\frac{2}{(L_1(r+\tau_0)+L_0)T d\sigma^2}}$,  $\sigma=c_2B\sqrt{T\log\frac{1}{\delta}}/(N\epsilon)$ and $T\ge\gO(N^2\epsilon^2/(B^2r^3d \log\frac{1}{\delta}))$. Then the condition in Lemma \ref{lemma:verify condition general sigma} holds as long as we have enough samples i.e. $N\geq L_1C\left(\epsilon,\delta,\tau,L,B,d,r\right)$.
\end{Lemma}

\begin{proof}
It is more straight-forward to verify the condition \eqref{eq:condition for second-order lemma} directly.
Firstly, we plug $\sigma^2=\frac{c_2^2B^2T\log(1/\delta)}{N^2\epsilon^2}$ from Lemma \ref{thm:privacy guarantee} into the formula of $\eta$ to have
\begin{equation*}
    \eta=\sqrt{\frac{2}{(L_1(r+\tau_0)+L_0)T d\sigma^2}}=\frac{N\epsilon}{c_2BT}\sqrt{\frac{2}{(L_1(r+\tau_0)+L_0)d\log(1/\delta)}}\leq \frac{\alpha_0N^2\epsilon^2}{6L_1dc_2^2B^2T\log(1/\delta)} =\frac{\alpha_0}{6L_1d\sigma^2}
\end{equation*}
as long as we have enough samples
\begin{equation*}
    N\geq \frac{6c_2BL_1}{\epsilon\alpha_0}\sqrt{\frac{2d\log(1/\delta)}{L_1(r+\tau_0)+L_0}}.
\end{equation*}
Other conditions
\begin{equation*}
    \eta=\frac{N\epsilon}{c_2BT}\sqrt{\frac{2}{(L_1(r+\tau_0)+L_0)d\log(1/\delta)}}\leq\min\left\{\dfrac{(r-\tau_0)\alpha_0}{4L_0},\dfrac{(1-\tau_1)\alpha_0}{4L_1} \right\}
\end{equation*}
holds as long as we run the algorithm long enough
\begin{equation}\label{eq:large T/N}
    \frac{T}{N}\geq\min\left\{\dfrac{L_0}{r-\tau_0},\dfrac{L_1}{\tau_1} \right\}\frac{4\epsilon}{\alpha_0c_2B}\sqrt{\frac{2}{(L_1(r+\tau_0)+L_0)d\log(1/\delta)}}.
\end{equation}
The last requirement \eqref{eq:large T/N} naturally holds due to $T\ge\gO(N^2\epsilon^2/(B^2r^3d \log\frac{1}{\delta}))$.
\end{proof}

\begin{proof}[Proof of Corollary \ref{cor:trade-off}]
Moreover, we take the limit $T\ge\gO(N^2\epsilon^2/(B^2r^3d \log\frac{1}{\delta}))$ to derive the privacy-utility trade-off
\begin{equation*}
    \mathbb{E}\left[\min_{0\leq k<T}\Vert\nabla f(\vx_k)\Vert\right]\leq\mathcal{O}\left(\sqrt[4]{\frac{(D_f+1)^2(L_1(r+\tau_0)+L_0)dB^2\log(1/\delta)}{N^2\epsilon^2}}+\frac{8(r+2\tau_0)\tau_0^3}{r(r+\tau_0)(1-\tau_1)^3}\right),
\end{equation*}
ending the proof.
\end{proof}

\section{Proofs for DP-SGD}\label{app:dp-sgd}
In this section, we prove the convergence theorem for DP-SGD following the roadmap outlined in Section \ref{sec:B}. To start with, we extend \eqref{eq:descent inequality 1 clip} into

\begin{align}
    \mathbb{E}_k\left[f(\vx_{k+1})\right]-f(\vx_{k})\leq& -\eta\mathbb{E}_k\left[\bar{h}_k\right]\Vert\nabla f(\vx_k)\Vert^2- \eta\mathbb{E}_k\left[\langle \bar{h}_k\nabla f(\vx_k),\vg_k-\nabla f(\vx_k) \rangle\right]\notag\\
    &+ \frac{L_0+L_1\Vert \nabla f(\vx_k)\Vert}{2}\eta^2\left( \mathbb{E}_k\left[\bar{h}_k^2\left\Vert \nabla f(\vx_k) \right\Vert^2\right]+ 2\mathbb{E}_k\left[\bar{h}_k^2\langle \vg_k-\nabla f(\vx_k),\nabla f(\vx_k)\rangle \right]\right)\notag\\
    &+ \frac{L_0+L_1\Vert \nabla f(\vx_k)\Vert}{2}\eta^2\left(dc^2\sigma^2+\mathbb{E}_k\left[\bar{h}_k^2\left\Vert \vg_k-\nabla f(\vx_k) \right\Vert^2\right]\right).\label{eq:descent inequality 2 clip}
\end{align}

\subsection{Upper Bound Second-order Terms}
In the same spirit as Lemma \ref{lemma:second order upper bound}, we provide an upper bound for the second-order terms in the following lemma.

\begin{Lemma}\label{lemma:second order upper bound clip}
For any $0<\alpha<1$ to be determined explicitly later, if
\begin{equation}\label{eq:condition for second-order lemma clip}
    \eta\leq\min\left\{\dfrac{\alpha}{6L_1dc\sigma^2},\dfrac{\alpha(1-\tau_1)}{2L_0(1-\tau_1)+4L_1\tau_0},\dfrac{\alpha\tau_0(1-\tau_1)}{4c(L_0(1-\tau_1)+2L_1\tau_0)}\right\},
\end{equation}
then we have
\begin{align}
    \frac{L_0+L_1\Vert\nabla f(\vx_k)\Vert}{2}\eta^2dc^2\sigma^2\leq& \frac{L_1(c+\tau_0)+L_0}{2}\eta^2dc^2\sigma^2+\frac{\alpha \eta  \bar{h}_k}{4}\Vert\nabla f(\vx_k)\Vert^2,\label{eq:technical lemma 1 clip}\\
    (L_0+L_1\Vert\nabla f(\vx_k)\Vert)\eta^2 \bar{h}_k^2\langle \nabla f(\vx_k), \vg_k-\nabla f(\vx_k)\rangle \leq& \frac{2(L_0(1-\tau_1)+2L_1\tau_0)\tau_0^2}{(1-\tau_1)^3}\eta^2+\frac{\alpha \eta  \bar{h}_k}{4}\Vert\nabla f(\vx_k)\Vert^2,\label{eq:technical lemma 2 clip}\\
    \frac{L_0+L_1\Vert\nabla f(\vx_k)\Vert}{2}\eta^2 \bar{h}_k^2\Vert \vg_k-\nabla f(\vx_k) \Vert^2\leq& \frac{2(L_0(1-\tau_1)+2L_1\tau_0)\tau_0^2}{(1-\tau_1)^3}\eta^2 +\frac{\alpha \eta  \bar{h}_k}{4}\Vert\nabla f(\vx_k)\Vert^2,\label{eq:technical lemma 3 clip}\\
    \frac{L_0+L_1\Vert\nabla f(\vx_k)\Vert}{2}\eta^2 \bar{h}_k^2\Vert\nabla f(\vx_k) \Vert^2\leq&\frac{\alpha \eta \bar{h}_k}{4}\Vert\nabla f(\vx_k)\Vert^2.\label{eq:technical lemma 4 clip}
\end{align}
\end{Lemma}

\begin{proof}
In fact, this lemma can be proved in an obvious way by separating into different cases.
\begin{description}
    \item[(i)] If $\Vert\nabla f(\vx_k)\Vert\leq c+\tau_0$, it directly follows that $L_1\Vert\nabla f(\vx_k)\Vert\eta^2d\sigma^2/2\leq L_1(c+\tau_0)\eta^2dc^2\sigma^2/2$; otherwise, if $\Vert\nabla f(\vx_k)\Vert> c+\tau_0$, we know
    \begin{equation*}
        \bar{h}_k=\min\left\{1,\frac{c}{\Vert\vg_k\Vert}\right\}\geq\frac{c}{c+\Vert\vg_k\Vert}\geq\frac{c}{c+\tau_0+(\tau_1+1)\Vert\nabla f(\vx_k)\Vert}\geq\frac{c}{3\Vert\nabla f(\vx_k)\Vert}
    \end{equation*}
    therefore $\eta\leq \alpha/(6L_1dc\sigma^2)$ directly yields
    \begin{equation*}
        \frac{L_1\Vert\nabla f(\vx_k)\Vert\eta^2dc^2\sigma^2}{2}\leq\frac{\eta\alpha \bar{h}_k\Vert\nabla f(\vx_k)\Vert^2}{4}.
    \end{equation*}
    Then \eqref{eq:technical lemma 1 clip} follows from summing up these two cases.
    
    \item[(ii)] If $\Vert\nabla f(\vx_k)\Vert\leq 2\tau_0/(1-\tau_1)$, then $\Vert \vg_k-\nabla f(\vx_k)\Vert\leq\tau_0+\tau_1\Vert\nabla f(\vx_k)\Vert\leq2\tau_0/(1-\tau_1)$ and $\bar{h}_k\leq1$, which yield
    \begin{align*}
        (L_0+L_1\Vert\nabla f(\vx_k)\Vert)\eta^2\bar{h}_k^2\langle \nabla f(\vx_k), \vg_k-\nabla f(\vx_k)\rangle \leq& \frac{2(L_0(1-\tau_1)+2L_1\tau_0)\tau_0^2}{(1-\tau_1)^3}\eta^2,\\
        \frac{L_0+L_1\Vert\nabla f(\vx_k)\Vert}{2}\eta^2\bar{h}_k^2\Vert \vg_k-\nabla f(\vx_k) \Vert^2\leq& \frac{2(L_0(1-\tau_1)+2L_1\tau_0)\tau_0^2}{(1-\tau_1)^3}\eta^2.
    \end{align*}
    Otherwise, if $\Vert\nabla f(\vx_k)\Vert> 2\tau_0/(1-\tau_1)$, we note that $\Vert \vg_k-\nabla f(\vx_k)\Vert\leq\frac{1+\tau_1}{2}\Vert\nabla f(\vx_k)\Vert$ and $\Vert\vg_k\Vert\geq\frac{1-\tau_1}{2}\Vert\nabla f(\vx_k)\Vert$. Moreover,
    \begin{equation*}
        \bar{h}_k(L_0+L_1\Vert\nabla f(\vx_k)\Vert)\leq\frac{c(L_0+L_1\Vert\nabla f(\vx_k)\Vert))}{\Vert\vg_k\Vert}\leq\frac{2c(L_0+L_1\Vert\nabla f(\vx_k)\Vert))}{(1-\tau_1)\Vert \nabla f(\vx_k)\Vert}\leq\frac{cL_0}{\tau_0}+\frac{2cL_1}{1-\tau_1}.
    \end{equation*}
    Consequently, once $\eta\leq\dfrac{\alpha}{4}\dfrac{\tau_0(1-\tau_1)}{c(L_0(1-\tau_1)+2L_1\tau_0)}$, we have
    \begin{align*}
        &(L_0+L_1\Vert\nabla f(\vx_k)\Vert)\eta^2\bar{h}_k^2\langle \nabla f(\vx_k), \vg_k-\nabla f(\vx_k)\rangle\\
        \leq& \left(\frac{cL_0}{\tau_0}+\frac{2cL_1}{1-\tau_1}\right)\eta^2 \bar{h}_k\Vert\nabla f(\vx_k)\Vert^2\leq\frac{\eta\alpha \bar{h}_k}{4}\Vert\nabla f(\vx_k)\Vert^2,
    \end{align*}
    and
    \begin{align*}
        &\frac{L_0+L_1\Vert\nabla f(\vx_k)\Vert}{2}\eta^2\bar{h}_k^2\Vert \vg_k-\nabla f(\vx_k) \Vert^2\\
        \leq& \frac{1}{2}\max\left(\frac{cL_0}{\tau_0}+\frac{2cL_1}{1-\tau_1}\right)\eta^2 \bar{h}_k\Vert\nabla f(\vx_k)\Vert^2\leq\frac{\eta\alpha \bar{h}_k}{4}\Vert\nabla f(\vx_k)\Vert^2.
    \end{align*}
    We obtain \eqref{eq:technical lemma 2 clip} and \eqref{eq:technical lemma 3 clip} via summing up respective bounds for two cases.
    
    \item[(iii)] We firstly derive a bound on $\bar{h}_k(L_0+L_1\Vert\nabla f(\vx_k)\Vert)$. When $\Vert\nabla f(\vx_k)\Vert\leq2\tau_0/(1-\tau_1)$, we know $\bar{h}_k(L_0+L_1\Vert\nabla f(\vx_k)\Vert)\leq\frac{L_0(1-\tau_1)+2L_1\tau_0}{1-\tau_1}$. Otherwise, we know $\bar{h}_k(L_0+L_1\Vert\nabla f(\vx_k)\Vert)\leq\frac{cL_0}{\tau_0}+\frac{2cL_1}{1-\tau_1}$. The last bound \eqref{eq:technical lemma 4 clip} can be derived directly by
    \begin{align*}
        &\frac{L_0+L_1\Vert\nabla f(\vx_k)\Vert}{2}\eta^2\bar{h}_k^2\Vert\nabla f(\vx_k) \Vert^2\\
        \leq&\max\left(\frac{L_0(1-\tau_1)+2L_1\tau_0}{1-\tau_1},\frac{cL_0}{\tau_0}+\frac{2cL_1}{1-\tau_1}\right)\frac{\eta^2}{2} \bar{h}_k\Vert\nabla f(\vx_k)\Vert^2\leq\frac{\eta\alpha \bar{h}_k}{4}\Vert\nabla f(\vx_k)\Vert^2
    \end{align*}
    via setting $\eta\leq\dfrac{\alpha}{2}\min\left(\dfrac{1-\tau_1}{L_0(1-\tau_1)+2L_1\tau_0},\dfrac{\tau_0(1-\tau_1)}{c(L_0(1-\tau_1)+2L_1\tau_0)}\right)$.
\end{description}
In general, the four inequalities hold as long as we ensure \eqref{eq:condition for second-order lemma clip}.
\end{proof}

Explanations in Remarks \ref{remark:second order upper bound 1} and \ref{remark:second order upper bound 2} also explain the motivations behind this proof. In the sequel, we will use this lemma only with
\begin{equation}\label{eq:set alpha clip}
    \alpha=\alpha_0:=\frac{\tau_0(1-\tau_1)}{c(1-\tau_1)+2\tau_0}<\frac{1}{2}.
\end{equation}

\begin{Lemma}\label{lemma:verify condition general sigma clip}
In the statement of Theorem \ref{thm:main convergence clip}, we take $\eta=\sqrt{\frac{2}{(L_1(c+\tau_0)+L_0)T dc^2\sigma^2}}$. Then the condition \eqref{eq:condition for second-order lemma clip} in Lemma \ref{lemma:second order upper bound clip} holds as long as we run the algorithm long enough i.e. $T\geq C\left(\sigma^2,\tau,L,d,c\right)$.
\end{Lemma}

\begin{proof}
We see that
\begin{equation*}
    \eta=\sqrt{\frac{2}{(L_1(c+\tau_0)+L_0)T dc^2\sigma^2}}\leq\min\left\{\dfrac{\alpha_0}{6L_1dc\sigma^2},\dfrac{\alpha_0(1-\tau_1)}{2L_0(1-\tau_1)+4L_1\tau_0},\dfrac{\alpha_0\tau_0(1-\tau_1)}{4c(L_0(1-\tau_1)+2L_1\tau_0)}\right\}
\end{equation*}
is equivalent to
\begin{equation*}
    T\geq\frac{2}{(L_1(c+\tau_0)+L_0)dc^2\sigma^2}\max\left\{\dfrac{6L_1dc\sigma^2}{\alpha},\dfrac{2L_0(1-\tau_1)+4L_1\tau_0}{\alpha(1-\tau_1)},\dfrac{4c(L_0(1-\tau_1)+2L_1\tau_0)}{\alpha\tau_0(1-\tau_1)}\right\}^2.
\end{equation*}
\end{proof}

\subsection{Final Procedures in Proof}

\begin{proof}[Proof of Theorem \ref{thm:main convergence clip}]
Equipped with Lemmas \ref{lemma:first order lower bound clip}, \ref{lemma:second order upper bound clip} and \ref{lemma:verify condition clip}, we further wrote the one step inequality \eqref{eq:descent inequality 2 clip} into
\begin{align}
    \eta B(\Vert\nabla f(\vx_k)\Vert)\leq&f(\vx_{k})- \mathbb{E}_k\left[f(\vx_{k+1})\right]\notag\\
    &+\eta^2\left(\frac{(L_1(c+\tau_0)+L_0)dc^2\sigma^2}{2}+ \frac{4(L_0(1-\tau_1)+2L_1\tau_0)\tau_0^2}{(1-\tau_1)^3}\eta^2 \right).\label{eq:descent inequality 3 clip}
\end{align}
We then separate the time index into
\begin{equation*}
    \mathcal{U}=\left\{k<T:\Vert\nabla f(\vx_k)\Vert\geq \frac{\tau_0}{1-\tau_1}\right\}
\end{equation*}
and $\mathcal{U}^c=\{0,1,\cdots,T-1\}\backslash\mathcal{U}$. Given this, we derive from \eqref{eq:descent inequality 3 clip} that for any $k\in\mathcal{U}$,
\begin{align*}
    &\frac{\tau_0(c-1)}{c(1-\tau_1)+2\tau_0}\eta\Vert\nabla f(\vx_k)\Vert\\
    \leq&f(\vx_{k})-\mathbb{E}_k\left[f(\vx_{k+1})\right]+\eta^2\left(\frac{(L_1(c+\tau_0)+L_0)dc^2\sigma^2}{2}+ \frac{4(L_0(1-\tau_1)+2L_1\tau_0)\tau_0^2}{(1-\tau_1)^3}\eta^2 \right).
\end{align*}
Similarly, together with $\alpha_0\leq1/2$, \eqref{eq:first order lower bound 2} deduces that for any $k\in\mathcal{U}^c$,
\begin{align*}
     &\frac{1}{2}\eta\Vert\nabla f(\vx_k)\Vert^2 \\
    \leq&f(\vx_{k})-\mathbb{E}_k\left[f(\vx_{k+1})\right]+\eta^2\left(\frac{(L_1(c+\tau_0)+L_0)dc^2\sigma^2}{2}+ \frac{4(L_0(1-\tau_1)+2L_1\tau_0)\tau_0^2}{(1-\tau_1)^3}\eta^2 \right).
\end{align*}
Sum these inequalities altogether to have
\begin{align*}
    &\max\left\{\frac{\tau_0(c-1)}{c(1-\tau_1)+2\tau_0}\frac{1}{T}\sum_{k\in\mathcal{U}}\Vert\nabla f(\vx_k)\Vert,\frac{1}{2T}\sum_{k\notin\mathcal{U}}\Vert\nabla f(\vx_k)\Vert^2\right\}\\
    \leq&\frac{(f(x_0)-f^\ast)}{T\eta}+\eta\frac{(L_1(c+\tau_0)+L_0)dc^2\sigma^2}{2} +\eta \frac{4(L_0(1-\tau_1)+2L_1\tau_0)\tau_0^2}{(1-\tau_1)^3}.
\end{align*}
We minimize the sum of first two terms by setting
\begin{equation}\label{eq:set eta updated clip}
    \eta=\sqrt{\frac{2}{(L_1(c+\tau_0)+L_0)T dc^2\sigma^2}}.
\end{equation}
We then define
\begin{equation}
    \Delta=(D_f+1)\sqrt{\frac{(L_1(c+\tau_0)+L_0)dc^2\sigma^2}{2T}}+ \frac{4(L_0+2L_1\tau_0)\tau_0^2}{(1-\tau_1)^3}\sqrt{\frac{2}{(L_1(c+\tau_0)+L_0)T dc^2\sigma^2}}.
\end{equation}
Recall $\sigma^2$ can grow with $T$, so these two terms are $\mathcal{O}(\sigma/\sqrt{T}),\mathcal{O}(1/(\sigma\sqrt{T}))$ respectively. Then we further have
\begin{align}
    &\mathbb{E}\left[\min_{0\leq k<T}\Vert\nabla f(\vx_k)\Vert\right]\notag\\
    \leq&\mathbb{E}\left[\min\left\{\sqrt{\dfrac{1}{|\mathcal{U}|}\sum_{k\in\mathcal{U}^c}\Vert\nabla f(\vx_k)\Vert^2},\dfrac{1}{|\mathcal{U}|}\sum_{k\notin\mathcal{U}}\Vert\nabla f(\vx_k)\Vert\right\}\right]\notag\\
    \leq&\max\left\{\sqrt{4\Delta},\dfrac{2(c+2\tau_0)}{\tau_0(c-1)}\Delta\right\}\label{eq:final inequality clip},
\end{align}
where the second inequality follows from the fact that either $|\mathcal{U}|\geq T/2$ or $|\mathcal{U}^c|\geq T/2$.
In the end, we capture the leading terms in this upper bound to have
\begin{align*}
    \mathbb{E}\left[\min_{0\leq k<T}\Vert\nabla f(\vx_k)\Vert\right]\leq&\mathcal{O}\left(\sqrt[4]{\frac{(D_f+1)^2(L_1(c+\tau_0)+L_0)dc^2\sigma^2}{2T}}\right)\\
    +&\mathcal{O}\left(\sqrt{\sqrt{\frac{2}{(L_1(c+\tau_0)+L_0)T dc^2\sigma^2}}\frac{(L_0+2L_1\tau_0)\tau_0}{(1-\tau_1)^3}} \right).
\end{align*}
\end{proof}

\begin{Lemma}\label{lemma:verify condition clip}
In the statement of Corollary \ref{cor:trade-off clip}, we take $\eta=\sqrt{\frac{2}{(L_1(c+\tau_0)+L_0)T dc^2\sigma^2}}$, $\sigma=c_2B\sqrt{T\log(1/\delta)}/(N\epsilon)$ and $T\ge\gO(N^2\epsilon^2/(B^2c^3d \log\frac{1}{\delta}))$. Then the condition in Lemma \ref{lemma:verify condition general sigma clip} holds as long as we have enough samples i.e. $N\geq L_1C\left(\epsilon,\delta,\tau,L,B,d,c\right)$.
\end{Lemma}

\begin{proof}
It is more straight-forward to verify the condition \eqref{eq:condition for second-order lemma} directly.
Firstly, we plug $\sigma^2=\frac{c_2^2B^2T\log(1/\delta)}{N^2\epsilon^2}$ from Lemma \ref{thm:privacy guarantee} into the formula of $\eta$ to have
\begin{equation*}
    \eta=\sqrt{\frac{2}{(L_1(c+\tau_0)+L_0)T dc^2\sigma^2}}=\frac{N\epsilon}{c_2BT}\sqrt{\frac{2}{(L_1(c+\tau_0)+L_0)dc^2\log(1/\delta)}}\leq \frac{\alpha_0N^2\epsilon^2}{6L_1dc^2c_2^2B^2T\log(1/\delta)} =\frac{\alpha_0}{6L_1dc^2\sigma^2}
\end{equation*}
as long as we have enough samples
\begin{equation*}
    N\geq \frac{6L_1cc_2B}{\epsilon\alpha_0}\sqrt{\frac{2d\log(1/\delta)}{L_1(c+\tau_0)+L_0}}.
\end{equation*}
Other conditions
\begin{equation*}
    \eta=\frac{N\epsilon}{cc_2BT}\sqrt{\frac{2}{(L_1(c+\tau_0)+L_0)d\log(1/\delta)}}\leq\min\left\{\dfrac{\alpha_0(1-\tau_1)}{2L_0(1-\tau_1)+4L_1\tau_0},\dfrac{\alpha_0\tau_0(1-\tau_1)}{4c(L_0(1-\tau_1)+2L_1\tau_0)}\right\}
\end{equation*}
holds as long as we run the algorithm long enough
\begin{equation}\label{eq:large T/N clip}
    \frac{T}{N}\geq\min\left\{1,\dfrac{2c}{\tau_0}\right\}\frac{2(L_0(1-\tau_1)+2L_1\tau_0)\epsilon}{cc_2B\alpha_0(1-\tau_1)}\sqrt{\frac{2}{(L_1(c+\tau_0)+L_0)d\log(1/\delta)}}.
\end{equation}
The last requirement \eqref{eq:large T/N clip} naturally holds due to $T\ge\gO(N^2\epsilon^2/(B^2c^3d \log\frac{1}{\delta}))$.
\end{proof}
\begin{proof}[Proof of Corollary \ref{cor:trade-off clip}]

Moreover, we take the limit $T\ge\gO(N^2\epsilon^2/(B^2c^3d \log\frac{1}{\delta}))$ to derive the privacy-utility trade-off
\begin{equation*}
    \mathbb{E}\left[\min_{0\leq k<T}\Vert\nabla f(\vx_k)\Vert\right]\leq\mathcal{O}\left(\sqrt[4]{\frac{(D_f+1)^2(L_1(c+\tau_0)+L_0)dc^2B^2\log(1/\delta)}{N^2\epsilon^2}}\right),
\end{equation*}
ending the proof.
\end{proof}

\section{More Experiments}\label{app:more-exp}

\subsection{Fine-tuning Large Language Models}

We use the pretrained RoBERTa model~\cite{liu2019roberta}\footnote{The model and checkpoints can be found at {https://github.com/pytorch/fairseq/tree/master/examples/roberta}.}, which has 125M parameters (RoBERTa-Base). We fine-tune the full pre-trained models except the embedding layer for SST-2 classification task from the GLUE benchmark~\cite{wang2018glue}. We adopt the setting as in ~\cite{li2021large}: full-precision training with the batch size 1000  and the number of epochs 10.

\textbf{Hyperparameter choice:}   For privacy parameters, we use $\epsilon=8, \delta=\text{1e-5}$. With Renyi differential privacy accountant, this corresponds to setting the noise multiplier  $0.635$. We compare the behavior of DP-SGD and DP-NSGD.   For DP-SGD, we search the clipping threshold  $c$ from $\{0.1, 0.5, 2.5, 12.5, 50.0\}$ and the learning rate $lr$ from $\{\text{0.05, 0.1, 0.2, 0.4, 0.8, 1.6}\}$. For the DP-NSGD, we search the learning rate $lr$ over the same set of DP-SGD and the regularizer $r$ from $\{\text{1e-3,  1e-2, 1e-1, 1., 10.0}\}$.

\begin{figure}[ht]
\vspace{-2mm}
\begin{center}
\centerline{\includegraphics[width=0.5\columnwidth]{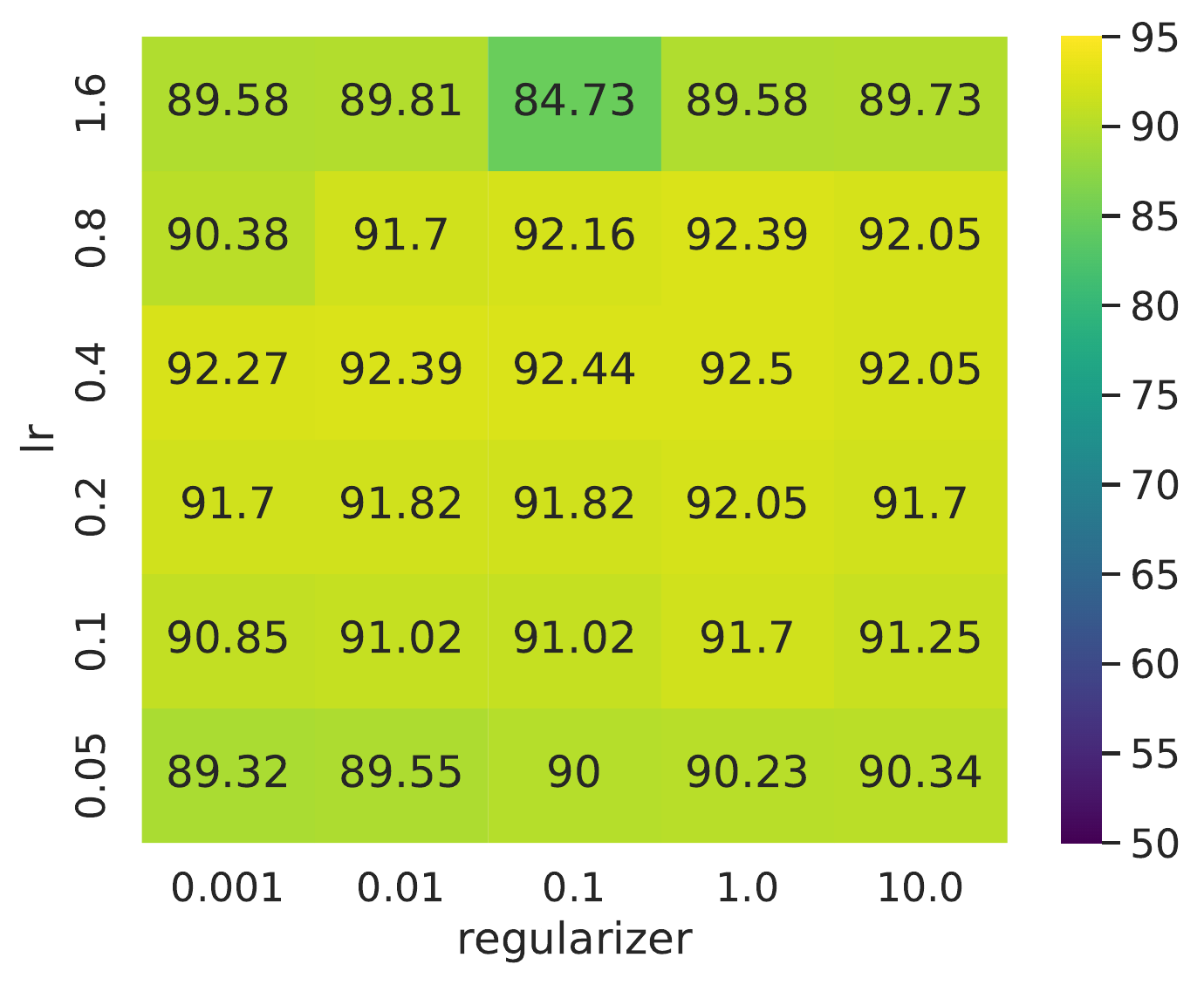}
\includegraphics[width=0.5\columnwidth]{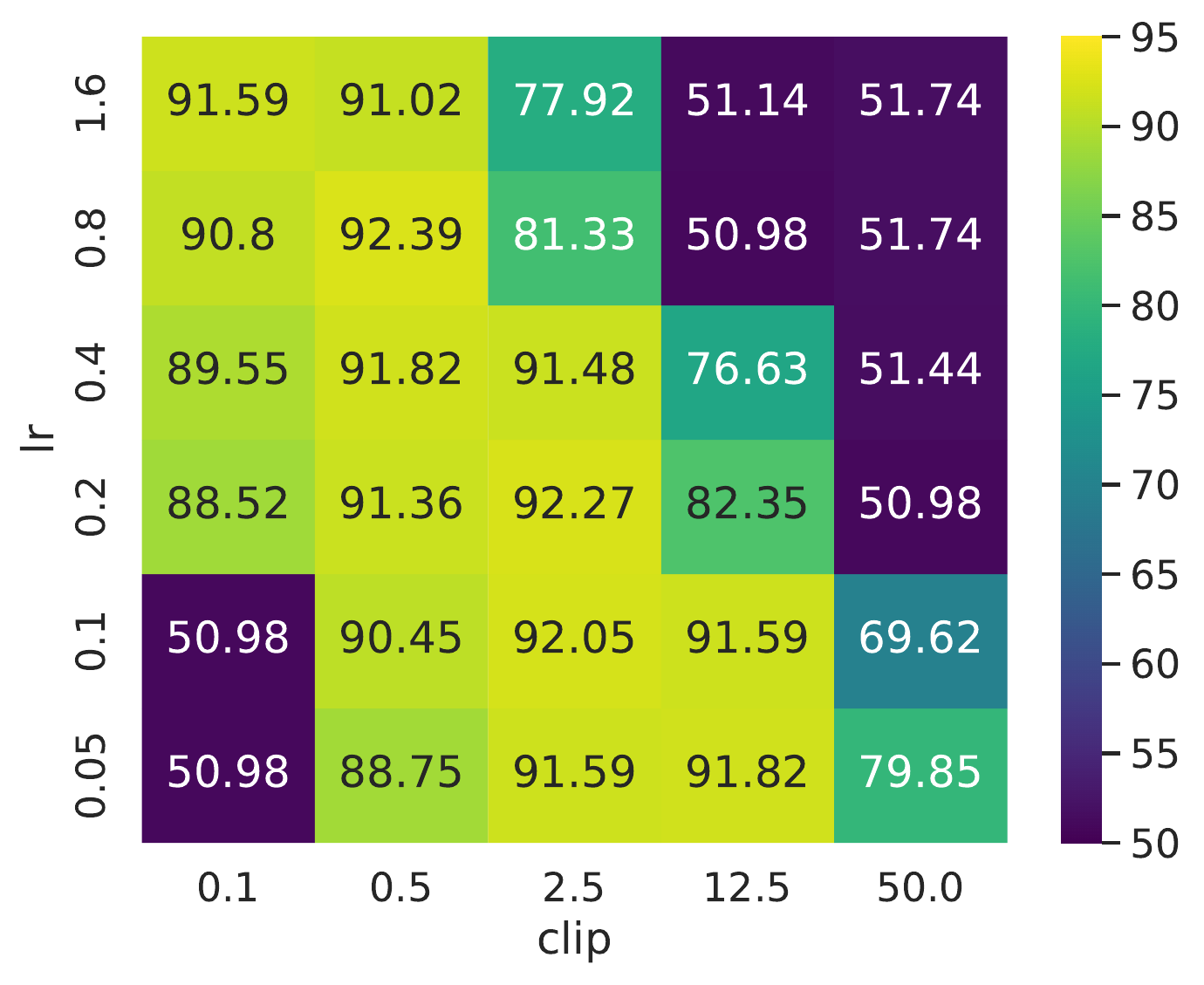}}
\caption{Experiments of fine-tuning RoBERTa on SST-2 task. Left: Accuracy heatmap of DP-NSGD with varying learning rates and regularizers. Right: Accuracy heatmap of DP-SGD with varying learning rates and clipping thresholds. }
\label{fig:heatmap-dpnsgd-dpsgd-sst2}
\end{center}
\vskip -0.2in
\end{figure}
We have similar observation in Figure \ref{fig:heatmap-dpnsgd-dpsgd-sst2} that the performance of DP-NSGD is rather stable for the regularizer and the learning rate, which indicates that it could be easier to tune than DP-SGD.

We also run the experiments with DP-Adam and DP-NAdam optimizers.  DP-Adam optimizer adds the per-example gradient clipping and Gaussian noise addition steps to Adam \cite{ADAM}.  For DP-Adam, we search the clipping threshold  $c$ takes values $\{0.1, 0.5, 2.5, 12.5, 50.0\}$ and the learning rate $lr$ taking values $\{\text{1e-4, 5e-4, 1e-3, 2e-3, 5e-3}\}$. For DP-NAdam, we use per-sample gradient normalization to replace the per-sample gradient clipping. We search the learning rate $lr$ from $\{\text{1e-4, 5e-4, 1e-3, 2e-3, 5e-3}\}$ and the regularizer factor $r$ from $\{\text{1e-3,  1e-2, 1e-1, 1., 10.0}\}$.

\begin{figure}[ht]
\vskip 0.2in
\begin{center}
\centerline{\includegraphics[width=0.5\columnwidth]{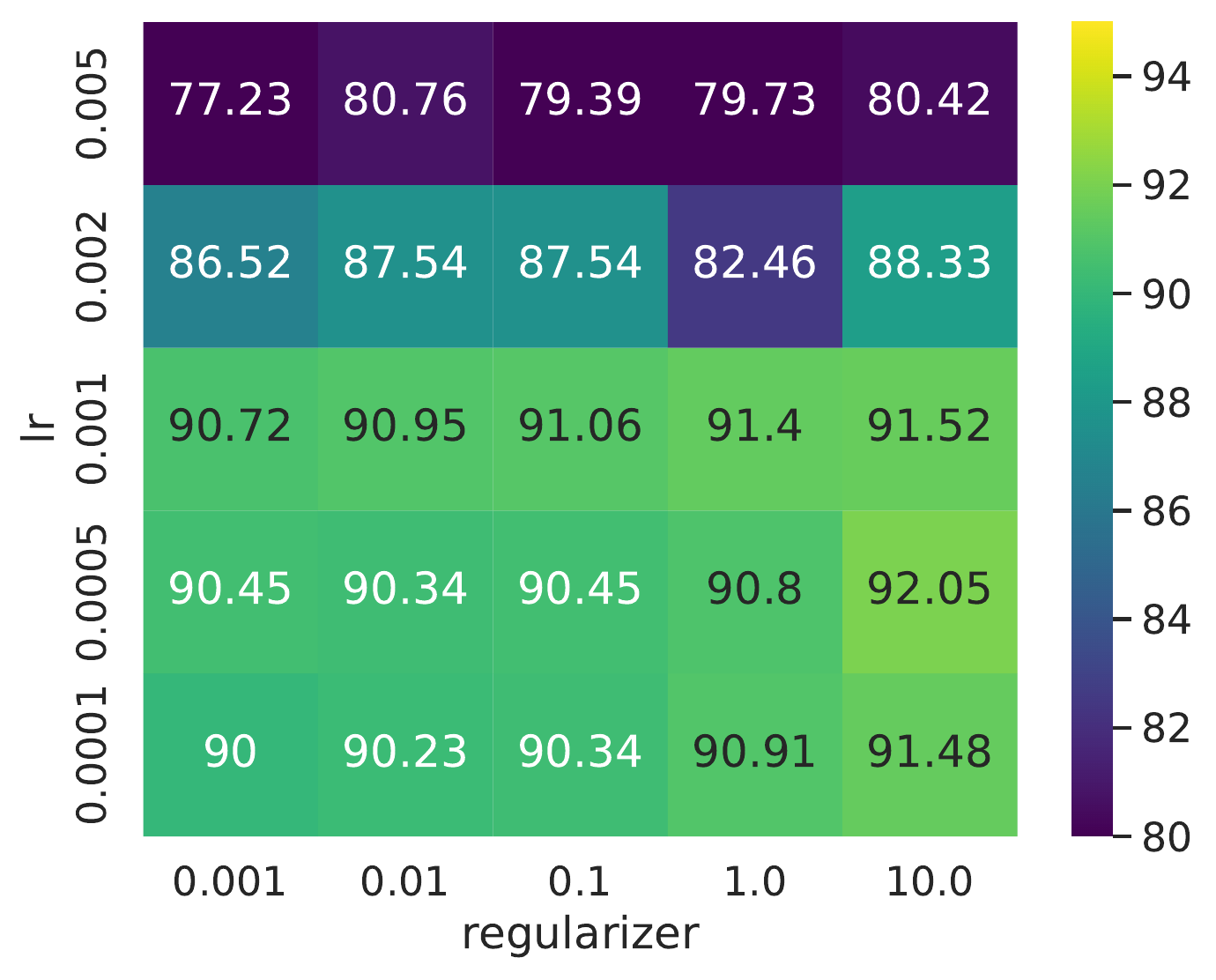}
\includegraphics[width=0.5\columnwidth]{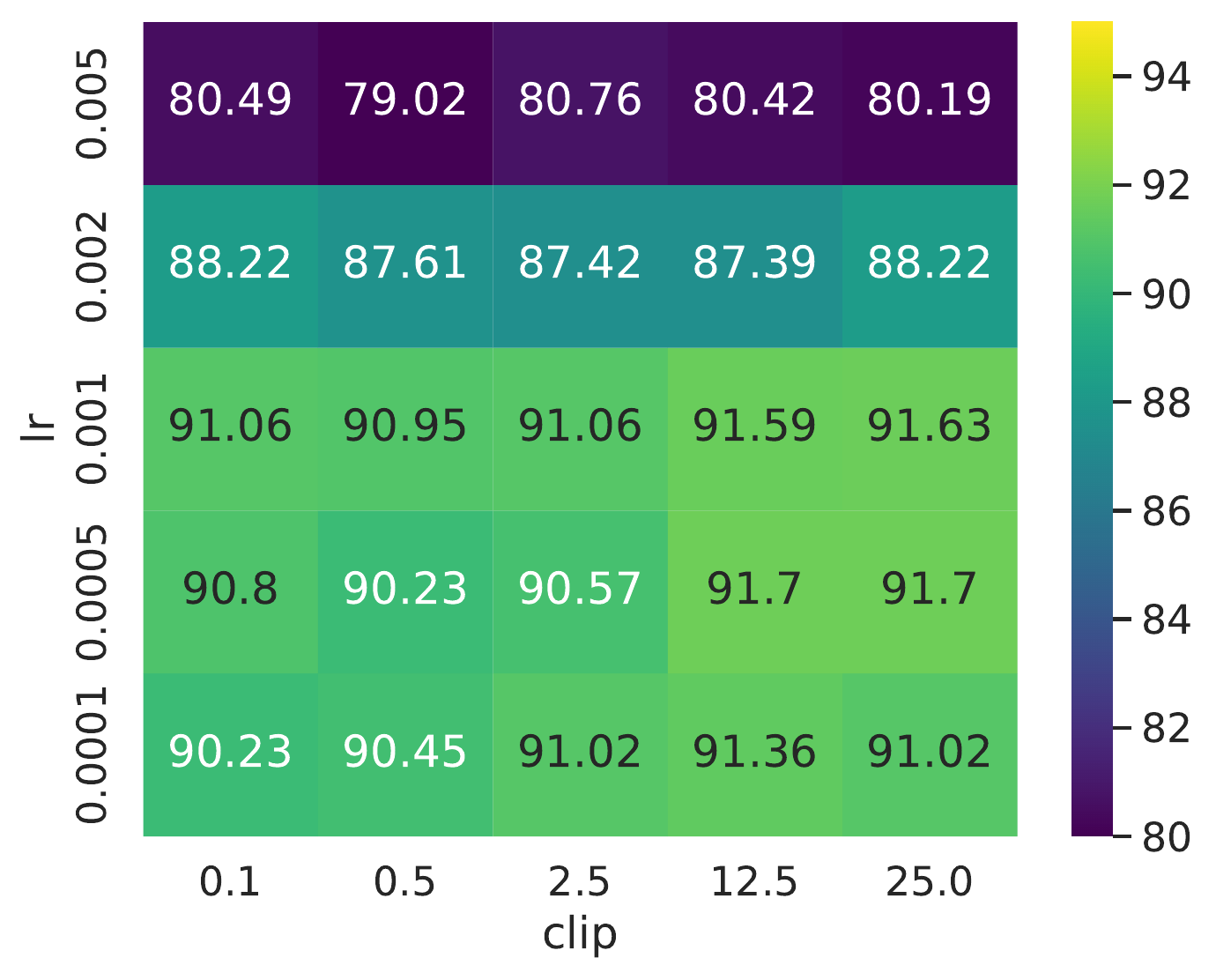}}
\caption{Left: Accuracy heatmap of DP-NAdam with varying learning rates and regularizers. Right: Accuracy heatmap of DP-Adam with varying learning rates and clips. }
\label{fig:heatmap-dpnadam-dpadam-sst2}
\end{center}
\vskip -0.2in
\end{figure}

\section{Proof for Privacy Guarantee}\label{app:privacy}

This section presents a simple proof for Lemma \ref{thm:privacy guarantee}. To begin with, we formally introduce the functional view of Renyi Differential Privacy below. Define a functional as
\begin{equation}\label{eq:privacy accountant functional}
    \epsilon_{\mathcal{M}}(\alpha)\triangleq\sup_{\mathbb{D},\mathbb{D}^\prime} D_\alpha(\mathcal{M}(\mathbb{D})\Vert\mathcal{M}(\mathbb{D}^\prime))=\sup_{\mathbb{D},\mathbb{D}^\prime}\frac{1}{\alpha-1}\log\mathbb{E}_{\theta\sim\mathcal{M}(\mathbb{D}^\prime)}\left[\left(\frac{\mathcal{M}(\mathbb{D})(\theta)}{\mathcal{M}(\mathbb{D}^\prime)(\theta)}\right)^\alpha\right],\alpha\geq1
\end{equation}
where $\mathcal{M}(\mathbb{D})$ denotes the distribution of the output with input $\mathbb{D}$ and $\mathcal{M}(\mathbb{D})(\theta)$ refers to the density at $\theta$ of this distribution. The following propositions clarify several notions of differential privacy in the literature.

\begin{Proposition}\label{prop:translation to DP}
Let $\mathcal{M}$ be a randomized mechanism.
\begin{description}
    \item[(i)] If and only if $\epsilon_\mathcal{M}(\infty)\leq\epsilon$, then $\mathcal{M}$ is $\epsilon$-(pure)-DP, \cite{dwork2014algorithmic}.
    \item[(ii)] If and only if $\epsilon_\mathcal{M}(\alpha)\leq\epsilon$, then $\mathcal{M}$ is $(\alpha,\epsilon)$-RDP (Renyi differential privacy), \cite{mironov2017renyi}.
    \item[(iii)] If and only if $\delta\geq\exp[(\alpha-1)(\epsilon_\mathcal{M}(\alpha)-\epsilon)]$ for some $\alpha\geq1$, then $\mathcal{M}$ is $(\epsilon,\delta)$-DP, \cite{dwork2014analyze}.
    \item[(iv)] If and only if $\epsilon_\mathcal{M}(\alpha)\leq\rho\alpha$ for any $\alpha\geq1$, then $\mathcal{M}$ is $\rho$-zCDP (zero-concentrated differential privacy), \cite{bun2016concentrated}.
    \item[(v)] If and only if $\epsilon_\mathcal{M}(\alpha)\leq\rho\alpha$ for any $\alpha\in(1,\omega)$, then $\mathcal{M}$ is $(\rho,\omega)$-tCDP (truncated concentrated differential privacy), \cite{bun2018composable}.
\end{description}
\end{Proposition}

We remark that Proposition \ref{prop:translation to DP}(iii) is adapted from the second assertion of Theorem 2 in \cite{abadi2016deep}, while the literature prefers to use the converse argument for this assertion, Proposition 3 in \cite{mironov2017renyi}. Here we also restate the composition theorem for Renyi differential privacy.

\begin{Proposition}[Proposition 1, \cite{mironov2017renyi}]\label{prop:composition theorem}
Let $\mathcal{M}=\mathcal{M}_T\circ\mathcal{M}_{T-1}\circ\cdots\circ\mathcal{M}_1$ be defined in an interactively compositional way, then for any fixed $\alpha\geq1$,
\begin{equation*}
    \epsilon_{\mathcal{M}}(\alpha)\leq\sum_{i=1}^{T}\epsilon_{\mathcal{M}_i}(\alpha).
\end{equation*}
\end{Proposition}

DP-SGD and DP-NSGD under our consideration can both be decomposed into $T$ composition of sub-sampled Gaussian mechanism with \textit{uniform sampling without replacement}, denoted as $\texttt{Gaussian}(\sigma)\circ\texttt{subsample}(N,B)$. We write the privacy-accountant functional, \eqref{eq:privacy accountant functional}, of this building-block mechanism as $\hat{\epsilon}(\alpha)$.

It is widely known that the sole Gaussian mechanism has $\epsilon_{\texttt{Gaussian}(\sigma)}(\alpha)=\alpha/(2\sigma^2)$, Table II in \cite{mironov2017renyi}. The sub-sampled Gaussian mechanism is much more complicated and draws many previous efforts. In particular, \cite{abadi2016deep,mironov2019renyi} study \textit{Poisson sub-sampling}, which is less popular in practical sub-sampling; \cite{wang2019subsampled} proposed a general bound for any \textit{uniformly} sub-sampled RDP mechanisms, but their bound is a bit loose when restricted to Gaussian mechanisms. Thankfully, \cite{bun2018composable} developed a general privacy-amplification bound for any \textit{uniformly} sub-sampled tCDP mechanisms, which is satisfying for our later treatment. Specifically, we specify Theorem 11 in \cite{bun2018composable} to the Gaussian mechanism, to get the following proposition.

\begin{Proposition}[Privacy Amplification by Uniform Sub-sampling without Replacement]\label{prop:sub-sampled Gaussian}
For the very mechanism $\texttt{Gaussian}(\sigma)\circ\texttt{subsample}(N,B)$ with $B<0.1N$, we have the following privacy accountant
\begin{equation}
    \hat{\epsilon}(\alpha)\leq \frac{7\gamma^2\alpha}{\sigma^2},\quad\forall\alpha\leq\frac{\sigma^2}{2}\log\left(\frac{1}{\gamma}\right),
\end{equation}
with $\gamma=B/N$.
\end{Proposition}

\begin{proof}[Proof of Lemma \ref{thm:privacy guarantee}]
We denote whole composited mechanism as $\mathcal{M}$. By Propositions \ref{prop:composition theorem} and \ref{prop:sub-sampled Gaussian}, we have
\begin{equation*}
    \epsilon_\mathcal{M}(\alpha)\leq\frac{7T\gamma^2\alpha}{\sigma^2},\quad\forall\alpha\leq\frac{\sigma^2}{2}\log\left(\frac{1}{\gamma}\right).
\end{equation*}
Further by Proposition \ref{prop:translation to DP}(iii), DP-SGD is $(\epsilon,\delta)$-DP if there exists $\alpha\leq\frac{\sigma^2}{2}\log\left(\frac{1}{\gamma}\right)$ such that
\begin{align*}
    7T\gamma^2\alpha/\sigma^2\leq&\epsilon/2,\\
    \exp(-(\alpha-1)\epsilon/2)\leq&\delta.
\end{align*}
Plus, we find that when $\epsilon=c_1\gamma^2T$, we can satisfy all these conditions by setting
\begin{equation*}
    \sigma\geq c_2\frac{\gamma\sqrt{T\log(1/\delta)}}{\epsilon}
\end{equation*}
for some explicit constants $c_1$ and $c_2$.
\end{proof}

\end{document}